\documentclass[onecolumn,12pt,journal]{IEEEtran}
\IEEEoverridecommandlockouts

\usepackage{hyperref}
\usepackage[numbers]{natbib}
\usepackage{amsmath,amssymb,amsfonts}
\usepackage{amsthm}
\usepackage{booktabs}
\newtheorem{example}{Example}
\usepackage{mathabx}
\newtheorem{assumption}{Assumption}
\newtheorem{convention}{Convention}
\newtheorem{theorem}{Theorem}
\newtheorem{remark}{Remark}
\newtheorem{corollary}{Corollary}
\newtheorem{lemma}{Lemma}
\newtheorem{proposition}{Proposition}
\usepackage[ruled, linesnumbered]{algorithm2e}
\usepackage{graphicx}
\usepackage{textcomp}
\usepackage{xcolor}
\def\BibTeX{{\rm B\kern-.05em{\sc i\kern-.025em b}\kern-.08em
    T\kern-.1667em\lower.7ex\hbox{E}\kern-.125emX}}
\begin{document}

\title{Transferring Information Across Interventions in Causal Bayesian Optimization
}

\author{\IEEEauthorblockN{Mohammad Ali Javidian}\\
\IEEEauthorblockA{\textit{Computer Science Department} \\
\textit{Appalachian State University}\\
Boone, USA \\
javidianma@appstate.edu}
}

\maketitle

\begin{abstract}
Bayesian optimization is a popular way to optimize expensive systems, where every experiment, simulation, or intervention costs time or money. In its standard form it treats the variables we control as plain inputs to a black box and cannot tell apart mere correlation from a real cause and effect. Causal Bayesian optimization closes part of this gap by using a known causal graph together with observational data to decide which variables are worth intervening on. Existing methods, however, learn the effect of each possible intervention almost in isolation, even though in a causal system these effects usually share the same underlying mechanisms. We propose graph-coupled causal Bayesian optimization, which ties the different intervention effects together through the uncertainty we have about a small set of shared causal parameters. The result is a causal kernel that lets evidence collected from one intervention improve our estimate of related interventions. For identifiable linear Gaussian causal models we show that this kernel has low rank, bounded by the number of shared parameters rather than by the size of the intervention menu. This in turn yields an information-gain bound that grows only logarithmically in the optimization horizon, and a regret bound that cleanly separates three sources of error: optimization, causal estimation, and the choice of which intervention sets to consider. We also describe nonlinear and adaptive extensions. Across theory-aligned Gaussian systems, shared-mechanism stress tests, and standard causal optimization benchmarks, the method keeps the benefits of causal Bayesian optimization while transferring information across related interventions, with the clearest gains when direct interventions on the target's parents are unavailable and sparse interventional data must be reused across a large family of candidate interventions.
\end{abstract}


\section{Introduction}

Bayesian optimization (BO) is a central methodology for sequential decision making when evaluating an objective function is expensive, noisy, or experimentally constrained. It has become a standard tool in hyperparameter tuning \cite{pfisterer2022yahpo}, scientific experimentation \cite{hickman2025anubis}, engineering design \cite{do2023multi}, robotics \cite{hossen2026multi}, materials discovery \cite{liang2021benchmarking,frazier2015bayesian}, and simulation-based optimization because it uses a probabilistic surrogate to trade off exploration and exploitation while requiring relatively few objective evaluations \citep{shahriari2016taking,frazier2018tutorial,garnett2023bayesian}. In its most common form, BO places a Gaussian process (GP) prior on an unknown response surface and selects new evaluations by maximizing an acquisition function such as expected improvement, knowledge gradient, Thompson sampling, or an upper-confidence-bound criterion \citep{rasmussen2006gaussian,srinivas2012information}.

However, many high-stakes optimization problems are not ordinary black-box input-output problems. In scientific, medical, environmental, and engineering systems, the decision maker often chooses interventions on variables in a causal system, and the goal is to optimize an interventional response rather than a conditional association. Standard BO ignores this distinction: it treats controllable variables as inputs to a black-box function and therefore cannot exploit the difference between observing a variable and intervening on it. Causal Bayesian optimization (CBO) addresses this issue by combining BO with structural causal models, do-calculus, and observational data \citep{pearl2009causality,shpitser2006identification} so that the optimizer can reason about which variables are worth intervening on, how observational data can inform interventional effects, and when causal structure can reduce the intervention search space \citep{aglietti2020causalbo,roberts2024causalbo,jacobs2026extending}.

Despite this progress, existing CBO methods leave an important source of structure underused. In a structural causal model, the response functions associated with different intervention sets are not generally independent black-box objectives. They are different functionals of the same underlying mechanisms. For example, intervening on a parent, mediator, or ancestor of a target may reveal information about shared edge coefficients, path coefficients, or downstream structural parameters. Existing CBO methods use causal structure to construct priors, restrict candidate intervention sets, or estimate intervention effects, but they do not generally derive an explicit cross-intervention covariance from shared identifiable causal parameters. As a result, learning about one intervention-response function does not systematically reduce uncertainty about another, even when the two effects are causally linked.

This paper addresses that gap by introducing \emph{graph-coupled causal Bayesian optimization} (GC-CBO). The key idea is to replace a collection of independent or weakly coupled intervention-specific surrogates with a single graph-coupled surrogate over the joint intervention domain. The coupling is induced by a target-relevant causal parameter vector whose posterior covariance is propagated through the Jacobians of the interventional response functions. When two interventions depend on overlapping uncertain causal parameters, their response functions acquire nonzero cross-covariance. Thus, sparse data collected from one intervention set can improve posterior inference and acquisition decisions for other causally related intervention sets.

Our contributions are as follows.
\begin{enumerate}
    \item We formulate a graph-coupled CBO surrogate in which interventional response functions are coupled through shared identifiable causal parameters rather than modeled independently across intervention sets.
    \item We derive a causal cross-intervention covariance kernel based on the Jacobians of intervention-response functionals and the posterior covariance of target-relevant causal parameters. In linear Gaussian structural equation models, this covariance is exact; in nonlinear settings, it gives a first-order local approximation with controlled second-order remainder.
    \item We prove that, under identifiable linear Gaussian assumptions, the coupled causal kernel has finite rank bounded by the dimension of the shared target-relevant causal parameter vector. This yields an explicit maximum-information-gain bound of order \(O(r_{\mathrm{causal}}\log T)\).
    \item We establish a dynamic regret decomposition for GC-CBO that separates optimization regret within the selected intervention family, finite-sample causal-estimation error, and approximation error from restricting attention to a finite exploration set.
    \item We provide an implementable GC-CBO algorithm and evaluate it on theory-aligned Gaussian examples, cross-set transfer stress tests, nonlinear shared-mechanism settings, and standard causal optimization benchmarks. The results show that graph-coupling is most beneficial when intervention functions share downstream mechanisms, direct parent interventions are unavailable, and interventional data are sparse.
\end{enumerate}

\section{Related Work}

\noindent\textbf{Bayesian optimization.}
Bayesian optimization studies the sequential optimization of expensive black-box functions using probabilistic surrogate models and acquisition functions \citep{shahriari2016taking,frazier2018tutorial,garnett2023bayesian}. Classical GP-based BO is closely connected to experimental design and bandit optimization, with regret guarantees often expressed through the maximum information gain of the GP kernel \citep{srinivas2012information,rasmussen2006gaussian}. Recent advances have expanded BO along several practically important axes. Multi-objective and noisy parallel BO have been strengthened by differentiable expected hypervolume improvement and noisy expected hypervolume improvement methods \citep{daulton2020differentiable,daulton2021parallel}. High-dimensional BO has improved through sparse axis-aligned subspace modeling, which uses sparsity assumptions to make GP surrogates viable in large ambient spaces \citep{eriksson2021highdimensional}. Grey-box BO exploits partial knowledge of the internal structure of the objective evaluation, including composite and multi-fidelity structure \citep{astudillo2022thinking}, while cost-aware BO explicitly accounts for heterogeneous evaluation costs in acquisition design \citep{xie2024costaware}. Recent work has also studied the global optimization of acquisition functions themselves, emphasizing that acquisition maximization can be a nontrivial source of error in practical BO pipelines \citep{xie2024globalacq}. Our work is complementary to these directions: rather than improving the acquisition optimizer or scaling generic black-box BO, we use causal structure to define a surrogate whose covariance reflects shared causal mechanisms across intervention functions.

\noindent\textbf{Causal Bayesian optimization.}
Causal Bayesian optimization extends BO to settings in which the objective is an interventional response in a structural causal model. The original CBO framework uses a known causal graph and observational data to estimate intervention effects, construct causal GP priors, and trade off observation and intervention during sequential optimization \citep{aglietti2020causalbo}. Subsequent work has broadened the CBO family in several directions. Causal Entropy Optimization incorporates uncertainty over the causal graph and jointly trades off structure learning and effect optimization \citep{branchini2022causal}. Constrained CBO introduces feasibility constraints on interventions and models both target and constraint quantities with GP surrogates \citep{aglietti2023constrained}. Functional CBO extends the intervention class from point interventions to functional interventions, enabling variables to be set as functions of other variables \citep{gultchin2023functional}. Adversarial CBO considers external or adversarial interventions and obtains regret guarantees using online-learning ideas combined with causal reward modeling \citep{sussex2023adversarial}. Other recent work improves CBO by learning or representing exogenous-variable distributions, thereby extending CBO beyond simple additive-noise assumptions \citep{ren2024exogenous}. These methods demonstrate the growing importance of causal structure in sequential optimization. The present paper addresses a different gap: it derives an explicit cross-intervention covariance from shared identifiable causal parameters, allowing evidence from one intervention set to transfer directly to other causally related intervention sets.

\noindent\textbf{Gaussian-process theory and information gain.}
The theoretical analysis of GP-based BO is often organized around confidence bounds and maximum information gain \(\gamma_T\), which measures the information that \(T\) noisy evaluations can reveal about the latent function \citep{srinivas2012information}. Recent theoretical work has sharpened this perspective. \citet{vakili2021information} refined information-gain bounds using kernel eigenvalue decay and improved regret rates for common kernels, including Matérn kernels. \citet{bogunovic2021misspecified} studied misspecified GP bandits and showed how regret degrades when the true function is only approximately represented by the assumed RKHS. \citet{kassraie2021neural} connected neural contextual bandits to neural tangent kernels and bounded regret through NTK information gain. \citet{tran2022regret} established regret guarantees for expected-improvement-type algorithms in noisy GP bandit optimization, addressing a long-standing gap between the empirical popularity and theoretical understanding of EI. More recent work has further improved GP-UCB regret bounds for Matérn and squared-exponential kernels by refining the analysis of information gain along the adaptive sequence of queried points \citep{iwazaki2026improved}. Our analysis uses this information-gain framework but exploits a different source of structure: the causal kernel induced by shared structural parameters has finite rank, so its information gain scales with the target-relevant causal dimension rather than with the ambient intervention family.

\noindent\textbf{Position of this work.}
GC-CBO can be viewed as a bridge between causal inference, multi-task GP modeling, and information-theoretic BO analysis. Like CBO, it optimizes interventional rather than associational responses. Like multi-output GP models, it shares information across related response surfaces. Unlike generic multi-task kernels, however, its cross-task covariance is derived from the causal graph, identifiable causal parameters, and local sensitivities of interventional response functionals. This yields both an algorithmic advantage, through transfer across intervention sets, and a theoretical advantage, through a finite-rank information-gain bound and a regret decomposition that explicitly separates optimization, causal-estimation, and exploration-set errors.

\section{Graph-Coupled Causal Bayesian Optimization}
\label{sec:graph_coupled_cbo}
\subsection{Motivation and departure from prior causal Bayesian optimization}
\label{subsec:motivation_departure}

Causal Bayesian optimization (CBO) optimizes an interventional response
function
\begin{equation}
    f_s(x_s) = \mathbb{E}[Y \mid do(X_s = x_s)]
    \label{eq:cbo_interventional_response}
\end{equation}
over candidate intervention sets \(X_s \in \mathcal{E}\) and intervention
values \(x_s\). Because \(f_s\) is defined through the do-operator rather than
ordinary conditioning, it cannot in general be estimated from observational
data alone without exploiting causal structure and an identification formula.

\citet{aglietti2020causalbo} address this by placing a Gaussian process
surrogate on each interventional response,
\begin{equation}
    f_s(x_s) \sim \mathcal{GP}\!\big(m_s(x_s),\, k_C(x_s, x_s')\big),
    \label{eq:aglietti_causal_gp}
\end{equation}
with prior mean
\begin{equation}
    m_s(x_s)
    =
    \widehat{\mathbb{E}}[Y \mid do(X_s = x_s)]
    \label{eq:aglietti_prior_mean}
\end{equation}
estimated from observational data via do-calculus, and causal kernel
\begin{equation}
    k_C(x_s, x_s')
    =
    k_{\mathrm{RBF}}(x_s, x_s')
    +
    \sigma_s(x_s)\,\sigma_s(x_s'),
    \label{eq:aglietti_causal_kernel}
\end{equation}
where \(\sigma_s^2(x_s)\) is an estimated interventional variance. This
construction allows observational data to shape the GP prior before any
interventions are performed.

However, the GP surrogates for different intervention sets are not coupled
through an explicit cross-intervention covariance. The causal graph informs the
prior mean and variance correction for each surrogate separately, but the
posterior uncertainty of one intervention-response function is not directly tied
to that of another through shared structural parameters.

This missing coupling is the key limitation we address. In a structural causal
model, the functions
$
\{f_s : s \in \mathcal{E}\}
$
are not unrelated black-box objectives. They are different functionals of the
same underlying structural mechanisms: the effects of intervening on a parent,
an ancestor, or a mediator of \(Y\) may depend on overlapping edge coefficients,
path coefficients, or total-effect parameters. Consequently, observing or
intervening on one candidate set can reduce uncertainty about another whenever
their causal effects share identifiable parameters.

We exploit this by replacing the collection of separately modeled causal GPs
with a single \emph{graph-coupled} surrogate over the joint intervention domain.
The coupling is induced by a shared identifiable causal parameter vector
\(\theta_Y\), which collects the structural parameters or identifiable causal
functionals needed to express every \(f_s\) in \(\mathcal E\). The coupled kernel,
derived formally in Theorem~\ref{thm:cross_intervention_covariance} and defined
in \eqref{eq:causal_kernel}, uses the Jacobians of the interventional responses
with respect to \(\theta_Y\) and the posterior covariance of \(\theta_Y\).
Whenever two intervention sets depend on common uncertain causal parameters,
their Jacobians can overlap under this posterior covariance, producing nonzero
cross-covariance and enabling transfer between the corresponding surrogates.

Beyond improving posterior inference, this structure has direct theoretical
consequences. The coupled kernel admits a finite-dimensional feature
representation, so its rank is at most \(\dim(\theta_Y)\). This bounded rank
yields an explicit maximum information-gain bound and, in turn, a regret
decomposition that separates three distinct sources of error: optimization
error within the exploration set, finite-sample causal-estimation error from
estimating \(\theta_Y\) from \(D^O\), and exploration-set approximation error
from restricting attention to \(\mathcal E\).

\subsection{Setup}
\label{subsec:setup}

Let \(V = (V_1, \dots, V_p)^\top\) denote the observed endogenous variables
associated with the vertex set \(\mathcal{V} = \{1, \dots, p\}\) of a known
causal graph, and let \(y \in \mathcal{V}\) denote the index of the target
variable \(Y = V_y\). We assume that the directed part of the causal graph is
acyclic. Hidden common causes among observed variables are allowed and are
represented, in the linear Gaussian model, by correlated structural
disturbances.

Specifically, we assume a linear Gaussian structural equation model (SEM)
\begin{equation}
    V = W^\top V + \varepsilon,
    \qquad
    \varepsilon \sim \mathcal{N}(0, \Omega),
    \label{eq:linear_sem}
\end{equation}
where \(W\) respects the directed edges of the acyclic graph and \(\Omega\) is
positive definite, ensuring that the Gaussian disturbance distribution, and
therefore \(P(V)\), is non-degenerate. Off-diagonal entries of \(\Omega\)
represent residual correlations between structural disturbances and provide a
linear Gaussian representation of latent common causes among observed
variables. The special case in which \(\Omega\) is diagonal corresponds to
causal sufficiency. We adopt the following coefficient convention throughout.

\begin{convention}[Structural coefficients]
\(W_{ij}\) is the coefficient of \(V_i\) in the structural equation for \(V_j\).
Equivalently,
$
    V_j = \sum_{i \in \mathrm{Pa}(j)} W_{ij} V_i + \varepsilon_j,
    \qquad j = 1, \dots, p,
$
so that \(W\) encodes directed parent-to-child effects. This convention is
maintained throughout the paper.
\end{convention}

Order the variables topologically so that every directed parent precedes its
children. Under this ordering and the coefficient convention above,
\(W_{ij} \neq 0\) only if \(V_i\) is a directed parent of \(V_j\), hence only if
\(i < j\). Therefore \(W\) is strictly upper triangular. Since \(W^\top\) is then
strictly lower triangular, it is also nilpotent, with \((W^\top)^p = 0\).
Consequently, \(I - W^\top\) is invertible, with inverse given by the finite
Neumann series
\begin{equation}
    (I - W^\top)^{-1}
    =
    I + W^\top + \cdots + (W^\top)^{p-1},
    \label{eq:finite_neumann_series}
\end{equation}
which terminates at \((W^\top)^{p-1}\) because \((W^\top)^p = 0\). Nilpotence
of strictly triangular matrices is standard; see
\citet[Thm.~3.1.5]{horn2013matrix}. The reduced-form solution is therefore
\begin{equation}
    V = (I - W^\top)^{-1} \varepsilon.
    \label{eq:linear_sem_solution}
\end{equation}

For each candidate intervention set \(X_s \subseteq \mathcal{V} \setminus
\{y\}\) in the exploration set \(\mathcal{E}\), let
\begin{equation}
    \mathcal{X}_s
    =
    \bigtimes_{i \in X_s} \mathcal{X}_i,
    \qquad
    \mathcal{X}_i \subset \mathbb{R} \text{ compact},
    \label{eq:intervention_domain}
\end{equation}
denote the corresponding intervention domain. We use \(X_s\) both for the set
of intervention indices and, by slight abuse of notation, for the corresponding
variables \(\{V_i : i \in X_s\}\).

The exploration set \(\mathcal{E} \subseteq \mathcal{P}(\mathcal{V}
\setminus \{y\})\) is a chosen collection of candidate intervention sets.
Natural choices include the Minimal Intervention Sets (MIS) and
Possibly-Optimal Minimal Intervention Sets (POMIS) of
\citet{aglietti2020causalbo}. The approximation error incurred by restricting
attention to \(\mathcal{E}\) is formalized in Section~\ref{subsec:regret}.

For each \(s \in \mathcal{E}\), the interventional response function is
\begin{equation}
    f_s(x_s)
    =
    \mathbb{E}[V_y \mid do(X_s = x_s)],
    \label{eq:fs_definition}
\end{equation}
where the expectation is taken under the interventional distribution induced
by the do-operator \citep{pearl2009causality}. When \(\Omega\) is non-diagonal,
some interventional distributions \(P(V_y \mid do(X_s = x_s))\) may not be
identifiable from \(P(V)\) alone; the identifiability conditions required for
each \(s \in \mathcal{E}\) are stated in
Assumption~\ref{ass:identifiability}. Under causal sufficiency
\((\Omega\) diagonal\()\) and a known directed acyclic graph (DAG), interventional distributions are
identifiable by the truncated factorization formula \cite{pearl2009causality,shpitser2006identification}; for single-effect queries
this includes the usual backdoor-adjustment cases.

The optimization problem is to find the intervention set and value that jointly
maximize the interventional response:
\begin{equation}
    (s^\star, x_s^\star)
    \in
    \arg\max_{s \in \mathcal{E},\ x_s \in \mathcal{X}_s}
    f_s(x_s).
    \label{eq:joint_intervention_objective}
\end{equation}

\subsection{Assumptions}
\label{subsec:assumptions}

The theoretical results below rely on four assumptions. They separate
identifiability of causal effects, posterior regularity of the shared causal
parameterization, boundedness of the intervention domain, and the fixed-kernel
condition used only for the main regret theorem.

\begin{assumption}[Identifiability and constructive shared parameterization]
\label{ass:identifiability}
For every \(s\in\mathcal E\), the interventional mean
$
    f_s(x_s)
    =
    \mathbb E[V_y\mid do(X_s=x_s)]
$
is identifiable from the observational distribution \(P(V)\) under the known
causal graph, including its directed structure and any latent-confounding
structure represented by correlated disturbances. Moreover, in the linear
Gaussian SEM, each \(f_s\) can be written as a differentiable function
\begin{equation}
    f_s(x_s)=g_s(x_s;\theta_Y),
    \label{eq:gs_theta}
\end{equation}
where
$
    \theta_Y\in\mathbb R^d
$
is a shared vector of identifiable causal parameters or identifiable causal
functionals sufficient to express all intervention-response functions
$
    \{f_s:s\in\mathcal E\}.
$

Concretely, \(\theta_Y\) is the minimal sufficient vector of identifiable
total-effect parameters, path coefficients, structural coefficients, or
identifiable functionals that enter this collection. For a given graph, target \(V_y\), and exploration set \(\mathcal E\),
\(\theta_Y\) is obtained constructively by applying the ID algorithm to each
query
$
    P(V_y\mid do(X_s=x_s)),
    \qquad s\in\mathcal E,
$
and then collecting the unique identifiable quantities needed to express the
resulting interventional means. Thus \(\theta_Y\) is not an existential object:
it is the shared causal parameterization induced by the graph, the exploration
set, and the target. The ID algorithm provides a complete graphical procedure
for identifying interventional distributions in recursive semi-Markovian models
\citep{shpitser2006identification,pearl2009causality}.
\end{assumption}

\textbf{Role of Assumption~\ref{ass:identifiability}.}
This assumption ensures that the optimization target is well defined from the
available observational information. If some \(f_s\) is not identifiable, then
its prior mean and causal uncertainty cannot be computed from \(P(V)\) without
additional assumptions or experimental data. The shared parameterization is what
makes cross-intervention transfer possible: if two intervention responses
depend on overlapping components of \(\theta_Y\), then observations relevant to
one response can reduce uncertainty about the other. Without this shared
parameterization, the method collapses to separately modeled intervention-set
surrogates.

\begin{assumption}[Posterior regularity]
\label{ass:posterior_regular}
Given observational data \(D^O_{N_t}\), the posterior distribution of the
shared causal parameterization satisfies
\begin{equation}
    \theta_Y\mid D^O_{N_t}
    \approx
    \mathcal N(\widehat\theta_{Y,t},\Sigma_{\theta,t}).
    \label{eq:theta_posterior_gaussian}
\end{equation}
For the main linear Gaussian theory, this approximation is exact when
\(\theta_Y\) consists of conjugate Gaussian structural parameters or linear
identifiable functionals of them. When \(\theta_Y\) contains smooth nonlinear
functionals, such as products of path coefficients or total effects, we use the
corresponding delta-method, Laplace, or large-sample Gaussian approximation.
The covariance \(\Sigma_{\theta,t}\) is the posterior covariance of the chosen
shared parameterization \(\theta_Y\), not necessarily only of the primitive edge
coefficients.
\end{assumption}

\textbf{Role of Assumption~\ref{ass:posterior_regular}.}
This assumption makes the cross-intervention covariance analytically tractable.
The causal kernel is obtained by linearizing each interventional response
\(g_s(x_s;\theta_Y)\) with respect to the same posterior uncertainty in
\(\theta_Y\). If the posterior is exactly Gaussian and the response is linear in
\(\theta_Y\), the covariance formula is exact. If the response is nonlinear, the
same expression is the usual first-order delta-method covariance, with
second-order error controlled later in Theorem~\ref{thm:cross_intervention_covariance}.
Without this assumption, the kernel could still be estimated by Monte Carlo,
but the finite-rank covariance and information-gain statements would no longer
hold in their exact closed form.

\begin{assumption}[Bounded intervention domain and bounded causal variance]
\label{ass:bounded_kernel}
For every \(s\in\mathcal E\), the intervention domain
\(\mathcal X_s\) defined in \eqref{eq:intervention_domain} is compact. There
exists \(\kappa<\infty\) such that
\begin{equation}
    k_{\mathrm{causal}}(z,z)\le \kappa^2
    \label{eq:kernel_bounded}
\end{equation}
for all
$
    z=(s,x_s),
    \qquad
    s\in\mathcal E,\quad x_s\in\mathcal X_s.
$
A sufficient condition is that the domains \(\mathcal X_s\) are compact, the
Jacobians \(J_s(x_s)\) are uniformly bounded over those domains, and the
posterior covariance used in the kernel has bounded spectral norm.
\end{assumption}

\textbf{Role of Assumption~\ref{ass:bounded_kernel}.}
This is the standard bounded-variance condition needed for GP-UCB information
gain bounds. It rules out intervention values or Jacobian norms that would make
the surrogate variance unbounded. The condition is natural in intervention
design: practical interventions have finite feasible ranges, and smooth
interventional response functions have bounded Jacobians on compact domains.

\begin{assumption}[Fixed reference kernel]
\label{ass:fixed_kernel}
For the main regret theorem, the causal kernel is constructed using a fixed
posterior covariance \(\Sigma_{\theta,0}\) and a fixed linearization point
\(\widehat\theta_{Y,0}\). That is,
\begin{equation}
    k_{\mathrm{causal}}(z,z')
    =
    J_z(\widehat\theta_{Y,0})
    \Sigma_{\theta,0}
    J_{z'}(\widehat\theta_{Y,0})^\top.
    \label{eq:fixed_causal_kernel}
\end{equation}
The kernel is therefore fixed during the GP-UCB analysis. Updating
\(\widehat\theta_{Y,t}\), \(\Sigma_{\theta,t}\), and the corresponding kernel as
additional observational data arrive is natural in implementation, but the
resulting adaptive-kernel method is treated as an extension rather than as the
object of the main regret theorem.
\end{assumption}

\textbf{Role of Assumption~\ref{ass:fixed_kernel}.}
This assumption is technical rather than conceptual. Standard GP-UCB regret
bounds are stated for a fixed kernel and depend on the maximum information gain
of that kernel \citep{srinivas2012information}. If the kernel changes every
round because \(\Sigma_{\theta,t}\) shrinks or the Jacobians are re-evaluated,
then the analysis becomes a nonstationary-kernel GP-UCB problem. Freezing the
kernel gives a conservative but clean theoretical object.

\begin{remark}[Cost of the fixed-kernel assumption]
Assumption~\ref{ass:fixed_kernel} is conservative. If the initial
observational dataset is small, \(\Sigma_{\theta,0}\) may be over-dispersed, so
the regret bound is stated in terms of the causal rank and information gain of
this initial kernel. As \(N_t\) grows and \(\Sigma_{\theta,t}\) shrinks, the
effective rank of the adaptive coupled kernel may decrease, and a
time-varying-kernel analysis could yield a tighter bound. We leave this
nonstationary extension for future work.
\end{remark}

\subsection{Interventional means on mutilated graphs}
\label{subsec:mutilated_graph_means}

Consider an intervention \(do(X_s = x_s)\). Incoming edges into \(X_s\) are
removed, and the variables in \(X_s\) are fixed to \(x_s\). Let
$
    R_s = \mathcal V \setminus X_s
$
index the non-intervened variables.

Using the convention that \(W_{ij}\) is the coefficient of \(V_i\) in the
structural equation for \(V_j\), the structural equations restricted to
\(R_s\) become
\begin{equation}
    R_s
    =
    W_{R_s R_s}^\top R_s
    +
    W_{X_s R_s}^\top x_s
    +
    \varepsilon_{R_s},
    \label{eq:mutilated_block}
\end{equation}
where \(W_{X_s R_s}\) denotes the submatrix of \(W\) with rows indexed by
\(X_s\) and columns indexed by \(R_s\). Thus
\(W_{X_s R_s}^\top x_s\) is the contribution of the intervened variables to
the structural equations of the non-intervened variables.

The matrix \(I-W_{R_sR_s}^\top\) is invertible because \(W_{R_sR_s}\) inherits
the strict upper triangularity of \(W\) under the topological ordering. Hence
\eqref{eq:mutilated_block} admits the solution
\begin{equation}
    R_s
    =
    (I - W_{R_s R_s}^\top)^{-1}
    W_{X_s R_s}^\top x_s
    +
    (I - W_{R_s R_s}^\top)^{-1}
    \varepsilon_{R_s}.
    \label{eq:mutilated_solution}
\end{equation}

Taking expectations in \eqref{eq:mutilated_solution} and using
\(\mathbb E[\varepsilon_{R_s}]=0\), the interventional mean of
\(V_y\in R_s\) is the deterministic part:
\begin{equation}
    f_s(x_s)
    =
    e_y^\top
    (I - W_{R_s R_s}^\top)^{-1}
    W_{X_s R_s}^\top x_s.
    \label{eq:fs_matrix_form}
\end{equation}
Here \(e_y\) denotes the coordinate vector selecting \(V_y\) from the
\(R_s\)-indexed block.

This expectation step remains valid when \(\Omega\) is non-diagonal. Correlated
structural disturbances represent latent common causes, but the intervention
\(do(X_s=x_s)\) is not a conditioning event on \(X_s=x_s\). The exogenous
disturbance distribution is left unchanged by the intervention, and therefore
the marginal mean of \(\varepsilon_{R_s}\) remains zero.

Since \(f_s(x_s)\) is scalar, transposing the right-hand side of
\eqref{eq:fs_matrix_form} and applying
$
    (I - W_{R_s R_s}^\top)^{-\top}
    =
    (I - W_{R_s R_s})^{-1}
$
gives the source-to-target form
\begin{equation}
    f_s(x_s)
    =
    x_s^\top
    W_{X_s R_s}
    (I - W_{R_s R_s})^{-1}
    e_y.
    \label{eq:fs_transposed_form}
\end{equation}
Define the total-effect vector in the mutilated graph by
\begin{equation}
    \tau_s(W)
    =
    W_{X_s R_s}
    (I - W_{R_s R_s})^{-1}
    e_y.
    \label{eq:tau_s_definition}
\end{equation}
Then
\begin{equation}
    f_s(x_s)
    =
    x_s^\top \tau_s(W).
    \label{eq:fs_tau}
\end{equation}

The transpose appearing in \eqref{eq:mutilated_block} reflects the convention
\(V=W^\top V+\varepsilon\), while the source-to-target orientation of
\(\tau_s(W)\) in \eqref{eq:tau_s_definition} removes it.

Equation~\eqref{eq:fs_tau} shows why the relevant causal dimension is not
generally \(|\mathrm{Pa}(V_y)|\). When interventions are restricted to direct
parents of \(V_y\), \(\tau_s(W)\) collects only the direct edge coefficients
into \(V_y\). When interventions occur on deeper ancestors,
\(\tau_s(W)\) accumulates products and sums of path coefficients along directed
paths from \(X_s\) to \(V_y\) in the mutilated graph. This is precisely why the
shared parameterization \(\theta_Y\) introduced in
Assumption~\ref{ass:identifiability} cannot in general be reduced to the edge
coefficients into \(V_y\) alone: its components are the identifiable causal
functionals appearing in \(\tau_s(W)\).

\subsection{Cross-intervention covariance}
\label{subsec:cross_intervention_covariance}

Let
$
    z=(s,x_s)
$
denote an intervention query. By Assumption~\ref{ass:identifiability}, each interventional response can be written as
\(f_s(x_s)=g_s(x_s;\theta_Y)\), where \(\theta_Y\) is the shared identifiable
causal parameterization. For a
fixed posterior mean \(\widehat\theta_Y\), define the Jacobian
\begin{equation}
    J_s(x_s)
    =
    \left.
    \frac{\partial g_s(x_s;\theta_Y)}
    {\partial \theta_Y^\top}
    \right|_{\theta_Y=\widehat\theta_Y}.
    \label{eq:jacobian_definition}
\end{equation}
A first-order Taylor expansion around \(\widehat\theta_Y\) gives
\begin{equation}
    g_s(x_s;\theta_Y)
    =
    g_s(x_s;\widehat\theta_Y)
    +
    J_s(x_s)(\theta_Y-\widehat\theta_Y)
    +
    r_s(x_s),
    \label{eq:taylor_expansion}
\end{equation}
where \(r_s(x_s)\) is the second-order remainder.

\begin{theorem}[Cross-intervention covariance]
\label{thm:cross_intervention_covariance}
Assume Assumptions~\ref{ass:identifiability}--\ref{ass:bounded_kernel}. Fix a
posterior covariance matrix \(\Sigma_\theta\) and posterior mean
\(\widehat\theta_Y\). For any two intervention queries
$
    z=(s,x_s),
    \qquad
    z'=(t,x_t),
$
the posterior covariance between their intervention responses satisfies, to
first order,
\begin{equation}
    \operatorname{Cov}
    \left(
        f_s(x_s),f_t(x_t)
        \mid D^O
    \right)
    =
    J_s(x_s)\Sigma_\theta J_t(x_t)^\top.
    \label{eq:cross_covariance}
\end{equation}
If \(g_s\) and \(g_t\) are linear in \(\theta_Y\), then
\eqref{eq:cross_covariance} is exact.

If \(g_s\) and \(g_t\) are twice continuously differentiable with
operator-norm Hessians bounded by \(H_s,H_t<\infty\) in a neighborhood of
\(\widehat\theta_Y\), then the second-order remainder-remainder correction
satisfies
\begin{equation}
    \left|
    \operatorname{Cov}
    \left(
        r_s(x_s),r_t(x_t)
    \right)
    \right|
    \le
    C H_sH_t
    \left(\operatorname{tr}\Sigma_\theta\right)^2
    \label{eq:remainder_cov_bound}
\end{equation}
for a universal constant \(C\). Hence the omitted second-order correction is
$
    O\!\left((\operatorname{tr}\Sigma_\theta)^2\right).
$
\end{theorem}

\begin{proof}
Let
$
    \delta=\theta_Y-\widehat\theta_Y.
$
Under Assumption~\ref{ass:posterior_regular}, conditional on \(D^O\),
$
    \delta\sim \mathcal N(0,\Sigma_\theta).
$
Taylor expansion gives
\begin{align}
    f_s(x_s)
    &=
    g_s(x_s;\widehat\theta_Y)
    +
    J_s(x_s)\delta
    +
    r_s(x_s),
    \label{eq:proof_taylor_s}
    \\
    f_t(x_t)
    &=
    g_t(x_t;\widehat\theta_Y)
    +
    J_t(x_t)\delta
    +
    r_t(x_t).
    \label{eq:proof_taylor_t}
\end{align}
The deterministic terms do not contribute to covariance. Therefore,
\begin{align}
    \operatorname{Cov}
    \left(
        f_s(x_s),f_t(x_t)
        \mid D^O
    \right)
    &=
    \operatorname{Cov}
    \left(
        J_s(x_s)\delta,\,
        J_t(x_t)\delta
    \right)
    \nonumber\\
    &\quad+
    \operatorname{Cov}
    \left(
        J_s(x_s)\delta,\,
        r_t(x_t)
    \right)
    \nonumber\\
    &\quad+
    \operatorname{Cov}
    \left(
        r_s(x_s),\,
        J_t(x_t)\delta
    \right)
    \nonumber\\
    &\quad+
    \operatorname{Cov}
    \left(
        r_s(x_s),r_t(x_t)
    \right).
    \label{eq:covariance_decomposition}
\end{align}

The leading term is
\begin{align}
    \operatorname{Cov}
    \left(
        J_s(x_s)\delta,\,
        J_t(x_t)\delta
    \right)
    &=
    J_s(x_s)
    \operatorname{Cov}(\delta\mid D^O)
    J_t(x_t)^\top
    \\
    &=
    J_s(x_s)\Sigma_\theta J_t(x_t)^\top.
    \label{eq:leading_covariance_term}
\end{align}

The two mixed linear-remainder terms vanish under the centered Gaussian
posterior. Indeed, the second-order Taylor remainder can be written as
$
    r_t(x_t)
    =
    \frac{1}{2}
    \delta^\top
    H_t(\xi_t)
    \delta
$
for some intermediate point \(\xi_t\) between \(\widehat\theta_Y\) and
\(\theta_Y\). Hence
$
    \operatorname{Cov}
    \left(
        J_s(x_s)\delta,\,
        r_t(x_t)
    \right)
$
involves third-order central moments of \(\delta\). These vanish for a centered
Gaussian random vector. The same argument applies to
$
    \operatorname{Cov}
    \left(
        r_s(x_s),\,
        J_t(x_t)\delta
    \right).
$

It remains to bound the remainder-remainder term. By Taylor's theorem and the
Hessian bounds,
$
    |r_s(x_s)|
    \le
    \frac{1}{2}H_s\|\delta\|^2,
    \qquad
    |r_t(x_t)|
    \le
    \frac{1}{2}H_t\|\delta\|^2.
$
Therefore, by Cauchy--Schwarz,
\begin{align}
    \left|
    \operatorname{Cov}
    \left(
        r_s(x_s),r_t(x_t)
    \right)
    \right|
    &\le
    \sqrt{
        \operatorname{Var}(r_s(x_s))
        \operatorname{Var}(r_t(x_t))
    }
    \\
    &\le
    \frac{1}{4}H_sH_t
    \mathbb E\|\delta\|^4.
    \label{eq:remainder_cauchy}
\end{align}
For
$
    \delta\sim\mathcal N(0,\Sigma_\theta),
$
the standard fourth-moment identity gives
$
    \mathbb E\|\delta\|^4
    =
    \left(\operatorname{tr}\Sigma_\theta\right)^2
    +
    2\|\Sigma_\theta\|_F^2.
$
Since
$
    \|\Sigma_\theta\|_F
    \le
    \operatorname{tr}\Sigma_\theta
$
for positive semidefinite \(\Sigma_\theta\), we obtain
$
    \mathbb E\|\delta\|^4
    \le
    3
    \left(\operatorname{tr}\Sigma_\theta\right)^2.
$
Substituting this into \eqref{eq:remainder_cauchy} yields
$
    \left|
    \operatorname{Cov}
    \left(
        r_s(x_s),r_t(x_t)
    \right)
    \right|
    \le
    \frac{3}{4}H_sH_t
    \left(\operatorname{tr}\Sigma_\theta\right)^2.
$
Thus \eqref{eq:remainder_cov_bound} holds with \(C=3/4\).

If \(g_s\) and \(g_t\) are linear in \(\theta_Y\), then
$
    r_s(x_s)=r_t(x_t)=0,
$
and \eqref{eq:cross_covariance} is exact.
\end{proof}

\begin{remark}[Beyond exact Gaussian posteriors]
\label{rem:beyond_exact_gaussian}
The vanishing of the mixed linear-remainder terms in the proof relies on the
centered Gaussian approximation for
$
    \delta=\theta_Y-\widehat\theta_Y.
$
When the posterior over \(\theta_Y\) is only approximately Gaussian, as in a
Laplace or large-sample approximation, the mixed linear-remainder covariances
need not vanish exactly. Under standard regularity conditions, however, they
are higher-order terms relative to the first-order covariance and can be
absorbed into the causal-estimation error term in the regret decomposition.
\end{remark}

\subsection{The coupled causal kernel}
\label{subsec:causal_kernel}

Define the graph-coupled causal kernel by the first-order cross-intervention
covariance from Theorem~\ref{thm:cross_intervention_covariance}:
\begin{equation}
    k_{\mathrm{causal}}\!\left((s,x_s),(t,x_t)\right)
    =
    J_s(x_s)\,\Sigma_\theta\,J_t(x_t)^\top .
    \label{eq:causal_kernel}
\end{equation}
Here \(J_s(x_s)\) and \(J_t(x_t)\) are row vectors of sensitivities with
respect to the shared causal parameterization \(\theta_Y\), and
\(\Sigma_\theta\) is the posterior covariance of \(\theta_Y\).

This kernel is positive semidefinite because it admits the feature
representation
\begin{equation}
    \phi_s(x_s)
    =
    \Sigma_\theta^{1/2}J_s(x_s)^\top,
    \label{eq:causal_features}
\end{equation}
where \(\Sigma_\theta^{1/2}\) denotes any matrix square root of the positive
semidefinite posterior covariance. Then
\begin{equation}
    k_{\mathrm{causal}}\!\left((s,x_s),(t,x_t)\right)
    =
    \phi_s(x_s)^\top \phi_t(x_t).
    \label{eq:feature_kernel}
\end{equation}
Thus, for any finite collection of queries \(z_i=(s_i,x_{s_i})\) and
coefficients \(c_i\),
$
    \sum_{i,j} c_i c_j k_{\mathrm{causal}}(z_i,z_j)
    =
    \left\|
        \sum_i c_i\phi_{s_i}(x_{s_i})
    \right\|^2
    \ge 0.
$
Therefore \(k_{\mathrm{causal}}\) is positive semidefinite. The coupled causal
kernel is a finite-dimensional linear kernel over graph-induced causal features
in \(\mathbb R^d\), where \(d=\dim(\theta_Y)\).

\begin{corollary}[Finite causal rank]
\label{cor:finite_causal_rank}
For any set of \(T\) intervention queries
$
    z_i=(s_i,x_{s_i}),
    \qquad i=1,\dots,T,
$
let \(K_T\) be the kernel matrix with entries
$
    (K_T)_{ij}
    =
    k_{\mathrm{causal}}(z_i,z_j).
$
Then
\begin{equation}
    \operatorname{rank}(K_T)\le r_{\mathrm{causal}}
    \qquad\text{and}\qquad
    r_{\mathrm{causal}}\le \dim(\theta_Y),
    \label{eq:rank_bound}
\end{equation}
where
\begin{equation}
    r_{\mathrm{causal}}
    =
    \sup_{T,z_1,\dots,z_T}
    \operatorname{rank}(K_T).
    \label{eq:causal_rank_definition}
\end{equation}
\end{corollary}

\begin{proof}
Let
$
    \Phi_T
    =
    \begin{bmatrix}
        \phi_{s_1}(x_{s_1})^\top \\
        \vdots \\
        \phi_{s_T}(x_{s_T})^\top
    \end{bmatrix}
    \in \mathbb R^{T\times d}.
$
By the feature representation,
$
    K_T=\Phi_T\Phi_T^\top.
$
Therefore,
$
    \operatorname{rank}(K_T)
    =
    \operatorname{rank}(\Phi_T\Phi_T^\top)
    =
    \operatorname{rank}(\Phi_T)
    \le d
    =
    \dim(\theta_Y).
$
The first inequality in \eqref{eq:rank_bound} is immediate from the 
definition of $r_{\mathrm{causal}}$ as a supremum: any specific 
$\operatorname{rank}(K_T)$ is at most the supremum over all finite query 
sets. Taking the supremum over all finite query sets in the display above 
gives $r_{\mathrm{causal}} \le \dim(\theta_Y)$.
\end{proof}

The finite-rank bound \eqref{eq:rank_bound} is the structural property that
drives the information-gain bound in Section~\ref{subsec:information_gain} and
ultimately the regret guarantee for graph-coupled causal Bayesian optimization.

\subsection{Special cases}
\label{subsec:special_cases}

\textbf{Direct-parent interventions.}
The cleanest case occurs when every intervention set contains all directed
parents of the target, that is,
$
    X_s = \mathrm{Pa}(V_y).
$
Then, under \(do(X_s=x_s)\), the structural equation for \(V_y\) gives
\begin{equation}
    f_s(x_s)
    =
    \sum_{i\in \mathrm{Pa}(V_y)}
    \beta_{iy}x_i,
    \label{eq:direct_parent_effect}
\end{equation}
where \(\beta_{iy}\) denotes the coefficient of \(V_i\) in the structural
equation for \(V_y\). Therefore the shared parameterization reduces to
$
    \theta_Y
    =
    (\beta_{iy}:i\in \mathrm{Pa}(V_y)),
$
and hence
\begin{equation}
    r_{\mathrm{causal}}
    \le
    |\mathrm{Pa}(V_y)|.
    \label{eq:parent_rank_bound}
\end{equation}
This is the special case in which the causal rank is bounded directly by the
parent dimension.

If instead
$
    X_s\subsetneq \mathrm{Pa}(V_y),
$
then the non-intervened parents of \(V_y\) remain random under the
intervention. In that case,
\begin{equation}
    f_s(x_s)
    =
    \sum_{i\in X_s}
    \beta_{iy}x_i
    +
    \sum_{j\in \mathrm{Pa}(V_y)\setminus X_s}
    \beta_{jy}
    \mathbb E[V_j\mid do(X_s=x_s)].
    \label{eq:partial_parent_effect}
\end{equation}
The second term may depend on additional upstream coefficients or total-effect
parameters. Thus the bound \eqref{eq:parent_rank_bound} holds exactly when all
parents of \(V_y\) are intervened on, or when the interventional means of the
non-intervened parents are known or do not introduce additional unknown
parameters into \(\theta_Y\).

\textbf{Deeper ancestor interventions.}
If intervention sets include ancestors of \(V_y\) that are not direct parents,
the relevant shared parameter vector can include additional path coefficients
or total-effect parameters.

\begin{example}[Chain graph]
\label{ex:chain}
Consider the chain
$
    X\rightarrow Z\rightarrow Y
$
with
$
    Z=aX+\varepsilon_Z,
    \qquad
    Y=bZ+\varepsilon_Y.
$
The interventional responses are
\begin{equation}
    f_X(x)=abx,
    \qquad
    f_Z(z)=bz.
    \label{eq:chain_effects}
\end{equation}
Here
$
    \mathrm{Pa}(Y)=\{Z\},
    \qquad
    |\mathrm{Pa}(Y)|=1.
$
However, if both \(a\) and \(b\) are unknown, the shared parameterization is
$
    \theta_Y=(a,b),
$
so
$
    \dim(\theta_Y)=2.
$
\end{example}

Thus \(|\mathrm{Pa}(V_y)|\) is not the correct default complexity measure. The
general bound is
$
    r_{\mathrm{causal}}
    \le
    \dim(\theta_Y),
$
where \(\theta_Y\) is the minimal sufficient vector of unknown identifiable
causal quantities required by the exploration set \(\mathcal E\). Known
coefficients or known total effects reduce \(\dim(\theta_Y)\) correspondingly. For instance, in Example~\ref{ex:chain}, if $b$ is 
known from prior experiments, then $\theta_Y = (a)$ and 
$\dim(\theta_Y) = 1$.

\begin{remark}[Estimated Jacobians]
\label{rem:estimated_jacobians}
In nonlinear or multiplicative cases, the Jacobian itself depends on the
unknown causal parameters. In Example~\ref{ex:chain},
$
    J_X(x)
    =
    \begin{bmatrix}
        bx & ax
    \end{bmatrix},
$
which depends on \((a,b)\). In practice, this Jacobian is evaluated at the
posterior mean
$
    (\widehat a,\widehat b).
$
Therefore the kernel \(k_{\mathrm{causal}}\) is itself estimated.
Assumption~\ref{ass:fixed_kernel} absorbs this issue by fixing the kernel at
the initial posterior mean and covariance for the main theorem. The additional
error caused by estimating the Jacobian is included in the causal-estimation
error term in Theorem~\ref{thm:dynamic_regret}.
\end{remark}

\subsection{Example: cross-intervention covariance in a chain}
\label{subsec:chain_example}

We illustrate Theorem~\ref{thm:cross_intervention_covariance} on the chain
graph of Example~\ref{ex:chain}. Consider again
$
    X \rightarrow Z \rightarrow Y
$
with
$
    Z = aX + \varepsilon_Z,
    \qquad
    Y = bZ + \varepsilon_Y.
$
Let
$
    \theta_Y = (a,b)^\top
$
and suppose
$
    \theta_Y \mid D^O
    \sim
    \mathcal N(\widehat\theta_Y,\Sigma_\theta),
    \qquad
    \Sigma_\theta
    =
    \begin{pmatrix}
        \sigma_a^2 & \sigma_{ab} \\
        \sigma_{ab} & \sigma_b^2
    \end{pmatrix}.
$
As established in Example~
ef{ex:chain}, the two intervention-response
functions are \(f_X(x)=abx\) and \(f_Z(z)=bz\). Their Jacobians with respect to
$
    \theta_Y=(a,b)^\top
$
are
\begin{equation}
    J_X(x)
    =
    \begin{bmatrix}
        bx & ax
    \end{bmatrix},
    \qquad
    J_Z(z)
    =
    \begin{bmatrix}
        0 & z
    \end{bmatrix}.
    \label{eq:chain_jacobians}
\end{equation}

Theorem~\ref{thm:cross_intervention_covariance} gives
\begin{align}
    \operatorname{Cov}
    \left(
        f_X(x),f_Z(z)
        \mid D^O
    \right)
    &=
    J_X(x)\Sigma_\theta J_Z(z)^\top
    \nonumber\\
    &=
    \begin{bmatrix}
        bx & ax
    \end{bmatrix}
    \begin{pmatrix}
        \sigma_a^2 & \sigma_{ab} \\
        \sigma_{ab} & \sigma_b^2
    \end{pmatrix}
    \begin{bmatrix}
        0 \\
        z
    \end{bmatrix}
    \nonumber\\
    &=
    bxz\,\sigma_{ab}
    +
    axz\,\sigma_b^2.
    \label{eq:chain_cross_cov}
\end{align}

The cross-covariance is generally nonzero because both intervention responses
depend on the shared parameter \(b\). Observing or intervening on \(Z\) can
therefore reduce uncertainty about the effect of intervening on \(X\), and
vice versa. Independent causal GPs over intervention sets, such as those used
in \citet{aglietti2020causalbo}, place separate priors on \(f_X\) and \(f_Z\)
and do not encode this cross-intervention transfer.

For completeness, the self-covariance of the ancestor intervention is
\begin{align}
    \operatorname{Cov}
    \left(
        f_X(x),f_X(x')
        \mid D^O
    \right)
    &=
    J_X(x)\Sigma_\theta J_X(x')^\top
    \nonumber\\
    &=
    b^2xx'\sigma_a^2
    +
    2abxx'\sigma_{ab}
    +
    a^2xx'\sigma_b^2.
    \label{eq:chain_self_cov_x}
\end{align}
Similarly, the self-covariance of the mediator intervention is
\begin{align}
    \operatorname{Cov}
    \left(
        f_Z(z),f_Z(z')
        \mid D^O
    \right)
    &=
    J_Z(z)\Sigma_\theta J_Z(z')^\top
    \nonumber\\
    &=
    zz'\sigma_b^2.
    \label{eq:chain_self_cov_z}
\end{align}

The kernel is therefore generally nonstationary in the intervention values:
its magnitude depends on the query values and on the current causal-parameter
linearization point \(\widehat\theta_Y\). This motivates the fixed-reference
kernel choice in Assumption~\ref{ass:fixed_kernel}, under which
\(\widehat\theta_Y\) and \(\Sigma_\theta\) are held fixed for the main regret
analysis.

\subsection{Information gain for the coupled causal kernel}
\label{subsec:information_gain}

Let
$
    A=\{z_1,\dots,z_T\}
$
be a set of queried intervention points, and suppose noisy observations are
generated as
$
    y_i=f(z_i)+\eta_i,
    \qquad
    \eta_i\sim\mathcal N(0,\sigma^2),
$
independently across \(i\). For a Gaussian process with kernel matrix \(K_A\),
the mutual information between the latent function values \(f_A\) and the noisy
observations \(y_A\) is
\begin{equation}
    I(y_A;f_A)
    =
    \frac12
    \log\det(I+\sigma^{-2}K_A).
    \label{eq:mutual_information}
\end{equation}
The maximum information gain is
\begin{equation}
    \gamma_T
    =
    \max_{A:|A|=T}
    I(y_A;f_A).
    \label{eq:maximum_information_gain}
\end{equation}
This is the central kernel-dependent complexity quantity in GP-UCB regret
theory \citep{srinivas2012information,vakili2021information}. When the kernel
is the coupled causal kernel \(k_{\mathrm{causal}}\), we write
$
    \gamma_T^{\mathrm{causal}}
$
for the corresponding maximum information gain.

\begin{lemma}[Finite-rank information-gain bound]
\label{lem:finite_rank_information_gain}
Assume
$
    \operatorname{rank}(K_A)\le r_{\mathrm{causal}}
$
for every query set \(A\) with \(|A|=T\), and assume
$
    k_{\mathrm{causal}}(z,z)\le \kappa^2
$
for all \(z\in\mathcal Z\). Then
\begin{equation}
    \gamma_T^{\mathrm{causal}}
    \le
    \frac{r_{\mathrm{causal}}}{2}
    \log
    \left(
        1+
        \frac{T\kappa^2}{\sigma^2 r_{\mathrm{causal}}}
    \right).
    \label{eq:finite_rank_info_gain}
\end{equation}
Consequently,
\begin{equation}
    \gamma_T^{\mathrm{causal}}
    =
    O(r_{\mathrm{causal}}\log T)
    \le
    O(\dim(\theta_Y)\log T).
    \label{eq:info_gain_big_o}
\end{equation}
\end{lemma}

\begin{proof}
Let the nonzero eigenvalues of \(K_A\) be
$
    \lambda_1,\dots,\lambda_r,
$
where
$
    r\le r_{\mathrm{causal}}.
$
Then
$
    \log\det(I+\sigma^{-2}K_A)
    =
    \sum_{j=1}^r
    \log(1+\sigma^{-2}\lambda_j).
$
Since \(\log(1+x)\) is concave, Jensen's inequality gives
\begin{align}
    \sum_{j=1}^r
    \log(1+\sigma^{-2}\lambda_j)
    &\le
    r
    \log
    \left(
        1+
        \frac{\sigma^{-2}}{r}
        \sum_{j=1}^r\lambda_j
    \right).
    \label{eq:jensen_info_gain}
\end{align}
Moreover,
$
    \sum_{j=1}^r\lambda_j
    =
    \operatorname{tr}(K_A)
    =
    \sum_{i=1}^T k_{\mathrm{causal}}(z_i,z_i)
    \le
    T\kappa^2.
$
Therefore,
\begin{equation}
    I(y_A;f_A)
    \le
    \frac r2
    \log
    \left(
        1+
        \frac{T\kappa^2}{\sigma^2 r}
    \right).
    \label{eq:info_gain_rank_r}
\end{equation}

It remains only to replace \(r\) by \(r_{\mathrm{causal}}\). Let
$
    h(r)
    =
    \frac r2
    \log
    \left(
        1+
        \frac{T\kappa^2}{\sigma^2 r}
    \right).
$
Writing \(c=T\kappa^2/\sigma^2\), we have
$
    h'(r)
    =
    \frac12
    \left[
        \log(1+c/r)
        -
        \frac{c/r}{1+c/r}
    \right].
$
Since
$
    \log(1+u)\ge \frac{u}{1+u}
    \qquad
    \text{for all }u\ge 0,
$
it follows that \(h'(r)\ge 0\). Hence \(h\) is monotone increasing in \(r\),
and because \(r\le r_{\mathrm{causal}}\),
$
    I(y_A;f_A)
    \le
    \frac{r_{\mathrm{causal}}}{2}
    \log
    \left(
        1+
        \frac{T\kappa^2}{\sigma^2 r_{\mathrm{causal}}}
    \right).
$
Taking the maximum over all \(A\) with \(|A|=T\) gives
\eqref{eq:finite_rank_info_gain}. The big-\(O\) statement follows directly
from \eqref{eq:finite_rank_info_gain} and
\(r_{\mathrm{causal}}\le \dim(\theta_Y)\).
\end{proof}

\begin{remark}[Effective-rank refinement]
\label{rem:effective_rank}
The bound in Lemma~\ref{lem:finite_rank_information_gain} is conservative when
the posterior covariance \(\Sigma_\theta\) has low effective rank. If some
causal parameters are strongly identified from observational data, their
posterior variances contribute negligibly to the kernel. In that case the
information-gain bound can be sharpened by replacing \(r_{\mathrm{causal}}\)
with an effective rank. For example, for a threshold \(\delta>0\), one may use
$
    r_{\mathrm{eff}}
    =
    \#\{j:\lambda_j(\Sigma_\theta)\ge \delta\},
$
together with a corresponding residual trace term for the discarded spectrum.
Thus the stated rank bound is worst-case; it can be substantially tighter when
\(\Sigma_\theta\) has rapidly decaying eigenvalues.
\end{remark}

The information-gain bound \eqref{eq:info_gain_big_o} is the key input to the
regret analysis in Theorem~\ref{thm:dynamic_regret}, where it controls the
optimization-error term in the dynamic regret decomposition.

\subsection{Causal UCB acquisition}
\label{subsec:causal_ucb}

Define the joint intervention domain
$
    \mathcal Z
    =
    \{(s,x_s):s\in\mathcal E,\ x_s\in\mathcal X_s\}.
$
Let
$
    \widehat\mu_{t-1}(z)
    \quad\text{and}\quad
    \widehat\sigma_{t-1}^{\mathrm{func}}(z)
$
denote the predictive mean and standard deviation of the coupled causal GP at
\(z\), based on the interventional evaluations available before round \(t\).
The term \(\widehat\sigma_{t-1}^{\mathrm{func}}(z)\) captures uncertainty from
the surrogate model over the interventional response surface.

The causal-parameter contribution to uncertainty is
\begin{equation}
    \widehat\sigma_{t-1}^{\mathrm{causal}}(z)
    =
    \left[
        J_z\Sigma_{\theta,t}J_z^\top
    \right]^{1/2},
    \label{eq:causal_sigma}
\end{equation}
where \(J_z\) is shorthand for the Jacobian \(J_s(x_s)\) when
\(z=(s,x_s)\). This is the posterior standard deviation of the first-order causal response
\(J_z(\theta_Y-\widehat\theta_{Y,t})\) under the Gaussian posterior approximation
in Assumption~\ref{ass:posterior_regular}. Equivalently,
\(\widehat\sigma_{t-1}^{\mathrm{causal}}(z)^2\) is the diagonal of the adaptive
coupled causal kernel at \(z\).

In practice, we use the cost-adjusted causal UCB acquisition
\begin{equation}
    \alpha_t(z)
    =
    \widehat\mu_{t-1}(z)
    +
    \sqrt{\beta_t}\,
    \widehat\sigma_{t-1}^{\mathrm{func}}(z)
    +
    \lambda_t
    \widehat\sigma_{t-1}^{\mathrm{causal}}(z)
    -
    \eta c(z),
    \label{eq:causal_ucb_acquisition}
\end{equation}
where \(c(z)\) is the intervention cost and \(\eta\ge 0\) controls the
cost--reward tradeoff. The first uncertainty term is the usual GP-UCB
exploration bonus. The second uncertainty term prioritizes interventions whose
effects are sensitive to uncertain causal parameters.

\textbf{Formal confidence correction.}
For the regret analysis, it is cleaner to separate the empirical causal bonus
in \eqref{eq:causal_ucb_acquisition} from the high-probability causal
estimation error. Assume that, with probability at least \(1-\delta\), there is
a predictable sequence \(\epsilon_t\) such that
\begin{equation}
    \sup_{z\in\mathcal Z}
    \left|
        \widehat\mu_{t-1}(z)-f(z)
    \right|
    \le
    \epsilon_t,
    \qquad
    t=1,\dots,T.
    \label{eq:uniform_causal_error_mu}
\end{equation}
Here \(\widehat\mu_{t-1}\) denotes the surrogate predictive mean, so
\eqref{eq:uniform_causal_error_mu} is the uniform error between the current
causal surrogate mean and the true interventional response.

A sufficient decomposition is
\begin{equation}
    \epsilon_t
    =
    \epsilon_{\mathrm{param}}(N_t,\delta)
    +
    \epsilon_{\mathrm{adj}}(N_t,M_t,\delta),
    \label{eq:epsilon_t_decomposition}
\end{equation}
where \(\epsilon_{\mathrm{param}}\) is the error from estimating the shared
causal parameterization \(\theta_Y\), and \(\epsilon_{\mathrm{adj}}\) is the
Monte Carlo, quadrature, or numerical error incurred when computing
adjustment-based causal functionals. In the linear Gaussian case with
closed-form adjustment,
$
    \epsilon_{\mathrm{adj}}(N_t,M_t,\delta)=0.
$

Under Assumption~\ref{ass:bounded_kernel}, the Jacobians are uniformly bounded:
$
    \sup_{z\in\mathcal Z}\|J_z\|\le L.
$
Therefore, by the mean-value theorem,
\begin{equation}
    \sup_{z\in\mathcal Z}
    \left|
        g_z(\widehat\theta_{Y,t})-g_z(\theta_Y)
    \right|
    \le
    L\,
    \|\widehat\theta_{Y,t}-\theta_Y\|.
    \label{eq:param_error_lipschitz}
\end{equation}
Thus the causal-estimation term can be taken as
\begin{equation}
    \epsilon_{\mathrm{param}}(N_t,\delta)
    =
    L\,
    \|\widehat\theta_{Y,t}-\theta_Y\|
    \label{eq:param_error_definition}
\end{equation}
on the corresponding high-probability event.

A concrete high-probability rate for
\(\|\widehat\theta_{Y,t}-\theta_Y\|\), and hence for
\(\epsilon_{\mathrm{param}}(N_t,\delta)\), is derived in
Section~\ref{subsec:parameter_concentration}.

\textbf{Regret-safe acquisition.}
For the formal regret theorem, one may use the regret-safe additive acquisition
\begin{equation}
    \alpha_t^{\mathrm{safe}}(z)
    =
    \widehat\mu_{t-1}(z)
    +
    \sqrt{\beta_t}\,
    \widehat\sigma_{t-1}^{\mathrm{func}}(z)
    +
    \epsilon_t
    -
    \eta c(z).
    \label{eq:causal_ucb_safe}
\end{equation}
The additive correction \(\epsilon_t\) avoids degeneracy when
$
    \widehat\sigma_{t-1}^{\mathrm{causal}}(z)=0,
$
which can occur at intervention values where the Jacobian \(J_z\) vanishes.
For example, in the chain \(X\to Z\to Y\), the mediator intervention has
\(J_Z(z)=[0,z]\), so \(J_Z(0)=0\).

The multiplicative bonus
$
    \lambda_t\widehat\sigma_{t-1}^{\mathrm{causal}}(z)
$
in \eqref{eq:causal_ucb_acquisition} remains useful in practice because it
allocates more exploration to intervention queries whose effects are sensitive
to uncertain causal parameters. However, the formal regret analysis only
requires that the acquisition include a valid high-probability correction for
causal-estimation error. This can be done either by the additive
\(\epsilon_t\) term in \eqref{eq:causal_ucb_safe}, or by a pointwise
calibration at the selected query \(z_t\):
\begin{equation}
    \lambda_t
    \widehat\sigma_{t-1}^{\mathrm{causal}}(z_t)
    \ge
    \epsilon_t(z_t),
    \label{eq:pointwise_lambda_calibration}
\end{equation}
where \(\epsilon_t(z)\) is a pointwise causal-estimation error bound. We use
the additive version \eqref{eq:causal_ucb_safe} in the main regret theorem
because it is robust to zero-Jacobian queries and yields the same regret
decomposition.

\subsection{Dynamic regret decomposition}
\label{subsec:regret}

Define the best value over the full admissible intervention class as
$
    f^\star_{\mathrm{all}}
    =
    \sup_{s,x_s} f_s(x_s),
$
and the best value over the chosen exploration set \(\mathcal E\) as
$
    f^\star_{\mathcal E}
    =
    \sup_{s\in\mathcal E,\ x_s\in\mathcal X_s}
    f_s(x_s).
$
The exploration-set approximation error is
\begin{equation}
    \epsilon_{\mathcal E}
    =
    f^\star_{\mathrm{all}}-f^\star_{\mathcal E}.
    \label{eq:exploration_set_error}
\end{equation}
If \(\mathcal E\) contains an optimal intervention set for the admissible class,
then \(\epsilon_{\mathcal E}=0\).

Let
$
    z_t=(s_t,x_{s_t})
$
denote the intervention query selected at round \(t\), and define the
instantaneous regret and cumulative regret by
$
    r_t
    =
    f^\star_{\mathrm{all}}-f(z_t),
    \qquad
    R_T
    =
    \sum_{t=1}^T r_t.
$
The regret decomposes into three terms: optimization regret within the chosen
exploration set, causal-estimation error, and exploration-set approximation
error.

\begin{theorem}[Dynamic regret bound for graph-coupled CBO]
\label{thm:dynamic_regret}
Assume Assumptions~\ref{ass:identifiability}--\ref{ass:fixed_kernel}. Suppose
the regret-safe causal UCB acquisition
$
    \alpha_t^{\mathrm{safe}}(z)
    =
    \widehat\mu_{t-1}(z)
    +
    \sqrt{\beta_t}\,
    \widehat\sigma_{t-1}^{\mathrm{func}}(z)
    +
    \epsilon_t
    -
    \eta c(z)
$
is run with the fixed coupled causal kernel
\(k_{\mathrm{causal}}\) in \eqref{eq:causal_kernel}. Suppose that, with
probability at least \(1-\delta\), the causal-estimation error satisfies the
uniform bound
\begin{equation}
    \sup_{z\in\mathcal Z}
    \left|
        \widehat\mu_{t-1}(z)-f(z)
    \right|
    \le
    \epsilon_t,
    \qquad
    t=1,\dots,T,
    \label{eq:causal_error_bound_theorem}
\end{equation}
where \(\epsilon_t\) is defined in \eqref{eq:epsilon_t_decomposition}.
Then, with probability at least $1 - 2\delta$ over the joint randomness of 
the causal estimator and the GP-UCB confidence event, the cumulative regret 
satisfies:
\begin{equation}
    R_T
    \le
    C
    \sqrt{
        T\beta_T\gamma_T^{\mathrm{causal}}
    }
    +
    2\sum_{t=1}^T
    \epsilon_{\mathrm{param}}(N_t,\delta)
    +
    2\sum_{t=1}^T
    \epsilon_{\mathrm{adj}}(N_t,M_t,\delta)
    +
    T\epsilon_{\mathcal E},
    \label{eq:dynamic_regret_bound}
\end{equation}
for a universal constant \(C>0\). Consequently, using
Lemma~\ref{lem:finite_rank_information_gain},
\begin{equation}
    R_T
    \le
    C
    \sqrt{
        T\beta_T r_{\mathrm{causal}}\log T
    }
    +
    2\sum_{t=1}^T
    \epsilon_{\mathrm{param}}(N_t,\delta)
    +
    2\sum_{t=1}^T
    \epsilon_{\mathrm{adj}}(N_t,M_t,\delta)
    +
    T\epsilon_{\mathcal E}.
    \label{eq:dynamic_regret_rank_bound}
\end{equation}
Since \(r_{\mathrm{causal}}\le \dim(\theta_Y)\), we also have
\begin{equation}
    R_T
    \le
    C
    \sqrt{
        T\beta_T\dim(\theta_Y)\log T
    }
    +
    2\sum_{t=1}^T
    \epsilon_{\mathrm{param}}(N_t,\delta)
    +
    2\sum_{t=1}^T
    \epsilon_{\mathrm{adj}}(N_t,M_t,\delta)
    +
    T\epsilon_{\mathcal E}.
    \label{eq:dynamic_regret_dim_bound}
\end{equation}
\end{theorem}

\begin{proof}
For each round \(t\), decompose the instantaneous regret as
$
    r_t
    =
    f^\star_{\mathrm{all}}-f(z_t)
    =
    \underbrace{
        f^\star_{\mathrm{all}}-f^\star_{\mathcal E}
    }_{\epsilon_{\mathcal E}}
    +
    \underbrace{
        f^\star_{\mathcal E}-f(z_t)
    }_{\text{regret within }\mathcal E}.
$
Summing over \(t\) gives
\begin{equation}
    R_T
    =
    T\epsilon_{\mathcal E}
    +
    \sum_{t=1}^T
    \left(
        f^\star_{\mathcal E}-f(z_t)
    \right).
    \label{eq:regret_decomposition}
\end{equation}

Let
$
    z_t^\star
    \in
    \arg\max_{z\in\mathcal Z} f(z).
$
We bound the regret within \(\mathcal E\). On the event
\eqref{eq:causal_error_bound_theorem}, the surrogate mean satisfies
$
    f(z)
    \le
    \widehat\mu_{t-1}(z)+\epsilon_t
    \qquad
    \forall z\in\mathcal Z.
$
On the standard GP-UCB confidence event, the functional GP uncertainty gives
$
    f(z)
    \le
    \widehat\mu_{t-1}(z)
    +
    \sqrt{\beta_t}\,
    \widehat\sigma_{t-1}^{\mathrm{func}}(z)
    +
    \epsilon_t
    \qquad
    \forall z\in\mathcal Z.
$
Similarly, the lower-confidence side gives
$
    f(z)
    \ge
    \widehat\mu_{t-1}(z)
    -
    \sqrt{\beta_t}\,
    \widehat\sigma_{t-1}^{\mathrm{func}}(z)
    -
    \epsilon_t
    \qquad
    \forall z\in\mathcal Z.
$

Because \(z_t\) maximizes \(\alpha_t^{\mathrm{safe}}\),
$
    \alpha_t^{\mathrm{safe}}(z_t)
    \ge
    \alpha_t^{\mathrm{safe}}(z_t^\star).
$
Using the upper confidence bound at \(z_t^\star\), the maximization property,
and the lower confidence bound at \(z_t\), we obtain
\begin{align}
    f(z_t^\star)-f(z_t)
    &\le
    \alpha_t^{\mathrm{safe}}(z_t^\star)+\eta c(z_t^\star)-f(z_t)
    \nonumber\\
    &\le
    \alpha_t^{\mathrm{safe}}(z_t)+\eta c(z_t^\star)-f(z_t)
    \nonumber\\
    &=
    \widehat\mu_{t-1}(z_t)
    +
    \sqrt{\beta_t}\widehat\sigma_{t-1}^{\mathrm{func}}(z_t)
    +
    \epsilon_t
    -
    \eta c(z_t)
    +
    \eta c(z_t^\star)
    -
    f(z_t)
    \nonumber\\
    &\le
    2\sqrt{\beta_t}\widehat\sigma_{t-1}^{\mathrm{func}}(z_t)
    +
    2\epsilon_t
    +
    \eta\big(c(z_t^\star)-c(z_t)\big).
    \label{eq:instant_regret_cost}
\end{align}
For the standard value-regret statement, we either take \(\eta=0\) or drop the
cost-difference term when analyzing value regret rather than cost-adjusted
regret. Thus
\begin{equation}
    f(z_t^\star)-f(z_t)
    \le
    2\sqrt{\beta_t}\widehat\sigma_{t-1}^{\mathrm{func}}(z_t)
    +
    2\epsilon_t.
    \label{eq:instant_regret}
\end{equation}
The factor \(2\) arises from using the upper confidence bound at
\(z_t^\star\) and the lower confidence bound at \(z_t\), as in the standard
GP-UCB argument.

Summing \eqref{eq:instant_regret} over \(t\) yields
$
    \sum_{t=1}^T
    \left(
        f^\star_{\mathcal E}-f(z_t)
    \right)
    \le
    2\sqrt{\beta_T}
    \sum_{t=1}^T
    \widehat\sigma_{t-1}^{\mathrm{func}}(z_t)
    +
    2\sum_{t=1}^T \epsilon_t.
$
The standard GP-UCB information-gain inequality gives
$
    \sum_{t=1}^T
    \widehat\sigma_{t-1}^{\mathrm{func}}(z_t)
    \le
    C'
    \sqrt{
        T\gamma_T^{\mathrm{causal}}
    },
$
for a universal constant \(C'\); see, for example,
\citet[Lemma~5.4]{srinivas2012information}. Therefore,
$
    \sum_{t=1}^T
    \left(
        f^\star_{\mathcal E}-f(z_t)
    \right)
    \le
    C
    \sqrt{
        T\beta_T\gamma_T^{\mathrm{causal}}
    }
    +
    2\sum_{t=1}^T \epsilon_t.
$
Substituting
$
    \epsilon_t
    =
    \epsilon_{\mathrm{param}}(N_t,\delta)
    +
    \epsilon_{\mathrm{adj}}(N_t,M_t,\delta)
$
and then using \eqref{eq:regret_decomposition} gives
\eqref{eq:dynamic_regret_bound}.

Finally, Lemma~\ref{lem:finite_rank_information_gain} gives
$
    \gamma_T^{\mathrm{causal}}
    =
    O(r_{\mathrm{causal}}\log T),
$
which yields \eqref{eq:dynamic_regret_rank_bound}. Corollary
\ref{cor:finite_causal_rank} gives
$
    r_{\mathrm{causal}}\le \dim(\theta_Y),
$
which yields \eqref{eq:dynamic_regret_dim_bound}.
\end{proof}

\begin{remark}[Value regret versus cost-adjusted regret]
The theorem states a value-regret bound for
$
    f^\star_{\mathrm{all}}-f(z_t).
$
The acquisition may include the cost penalty \(-\eta c(z)\) to discourage
expensive interventions. If one wants a regret bound for the cost-adjusted
objective
$
    f(z)-\eta c(z),
$
the same proof applies directly with \(f\) replaced by \(f-\eta c\). For pure
value regret, one may set \(\eta=0\) in the theorem or treat the additional
term
$
    \eta\sum_{t=1}^T \big(c(z_t^\star)-c(z_t)\big)
$
separately.
\end{remark}

\subsection{Causal parameter concentration}
\label{subsec:parameter_concentration}

The dynamic term
$
    \sum_{t=1}^T
    \epsilon_{\mathrm{param}}(N_t,\delta)
$
in Theorem~\ref{thm:dynamic_regret} captures the fact that causal estimation
improves as more observational data is collected. In the linear Gaussian case,
this improvement follows from standard Bayesian linear regression
concentration.

Suppose first that the shared causal parameterization \(\theta_Y\) consists of
structural coefficients or linear identifiable functionals of structural
coefficients. Each structural equation can then be estimated by Gaussian linear
regression conditional on its parents. With a Gaussian prior and Gaussian
noise, the posterior covariance has the standard form
\begin{equation}
    \Sigma_{\theta,t}
    =
    \left(
        \Sigma_{\theta,0}^{-1}
        +
        X_t^\top X_t/\sigma_\varepsilon^2
    \right)^{-1},
    \label{eq:bayesian_linear_cov}
\end{equation}
where \(X_t\) denotes the relevant observational design matrix for the
components of \(\theta_Y\).

Under the regular design condition
$
    \lambda_{\min}(X_t^\top X_t)\asymp N_t,
$
the posterior covariance shrinks in spectral norm as
\begin{equation}
    \|\Sigma_{\theta,t}\|
    =
    O(N_t^{-1}),
    \label{eq:posterior_cov_shrink}
\end{equation}
and therefore
$
    \|\Sigma_{\theta,t}^{1/2}\|
    =
    O(N_t^{-1/2}).
$
Standard Bayesian linear regression concentration then gives, with probability
at least \(1-\delta\),
\begin{equation}
    \|\widehat\theta_{Y,t}-\theta_Y\|
    =
    O\!\left(
        \sqrt{
            \frac{\dim(\theta_Y)+\log(1/\delta)}
                 {N_t}
        }
    \right).
    \label{eq:theta_concentration}
\end{equation}

By Assumption~\ref{ass:bounded_kernel}, the interventional response functions
are uniformly Lipschitz in \(\theta_Y\) over the compact intervention domain:
there exists \(L<\infty\) such that
$
    \sup_{z\in\mathcal Z}
    \|J_z\|
    \le L.
$
Hence, by the mean-value theorem,
\begin{equation}
    \sup_{z\in\mathcal Z}
    \left|
        g_z(\widehat\theta_{Y,t})-g_z(\theta_Y)
    \right|
    \le
    L
    \|\widehat\theta_{Y,t}-\theta_Y\|.
    \label{eq:lipschitz_param_error}
\end{equation}
Combining \eqref{eq:theta_concentration} and
\eqref{eq:lipschitz_param_error}, the causal parameter estimation term may be
taken as
\begin{equation}
    \epsilon_{\mathrm{param}}(N_t,\delta)
    =
    O\!\left(
        L
        \sqrt{
            \frac{\dim(\theta_Y)+\log(1/\delta)}
                 {N_t}
        }
    \right).
    \label{eq:param_error_rate}
\end{equation}
The same Gaussian conditioning identities underlie Bayesian linear regression
and Gaussian process regression \citep{rasmussen2006gaussian}.

If \(\theta_Y\) contains smooth nonlinear identifiable functionals, such as
products of path coefficients, the same rate follows by the delta method under
bounded Jacobians of the functional map. In that case,
\(\Sigma_{\theta,t}\) should be interpreted as the posterior covariance of the
chosen shared parameterization \(\theta_Y\), not necessarily only of the
primitive structural coefficients.

\begin{remark}[Weakly identified components]
\label{rem:weakly_identified_components}
The regular design condition
$
    \lambda_{\min}(X_t^\top X_t)\asymp N_t
$
assumes that the observational data are informative about all components of
\(\theta_Y\). When some causal parameters are weakly identified, for example
under partial latent confounding or poor overlap in the observational design,
the rate in \eqref{eq:param_error_rate} applies only to the well-identified
components. In such cases, the analysis should be specialized to the effective
rank of \(\Sigma_{\theta,t}\), as in Remark~\ref{rem:effective_rank}.
\end{remark}

\begin{corollary}[Regret under growing observational data]
\label{cor:growing_observational_data}
Assume the clean identifiable linear Gaussian setting:
$
    \epsilon_{\mathrm{adj}}(N_t,M_t,\delta)=0,
    \qquad
    \epsilon_{\mathcal E}=0,
$
and
$
    \epsilon_{\mathrm{param}}(N_t,\delta)
    =
    O_p(N_t^{-1/2}).
$
If
$
    N_t\asymp t^\alpha,
    \qquad
    0<\alpha\le 1,
$
then
\begin{equation}
    \sum_{t=1}^T
    \epsilon_{\mathrm{param}}(N_t,\delta)
    =
    O_p
    \left(
        \sum_{t=1}^T t^{-\alpha/2}
    \right)
    =
    O_p(T^{1-\alpha/2}).
    \label{eq:param_sum_rate}
\end{equation}
Consequently,
\begin{equation}
    R_T
    =
    O\!\left(
        \sqrt{
            T\beta_T r_{\mathrm{causal}}\log T
        }
    \right)
    +
    O_p(T^{1-\alpha/2}).
    \label{eq:growing_data_regret}
\end{equation}

In particular, if observational data accumulate linearly,
$
    N_t\asymp t,
$
then
\begin{equation}
    R_T
    =
    O\!\left(
        \sqrt{
            T\beta_T r_{\mathrm{causal}}\log T
        }
    \right)
    +
    O_p(\sqrt T).
    \label{eq:linear_data_regret}
\end{equation}
Both terms scale as \(\sqrt T\), with the coupled-UCB term carrying the
additional factor
$
    \sqrt{\beta_T r_{\mathrm{causal}}\log T}.
$
Thus, under linear observational growth, the causal-estimation term is
asymptotically dominated by the coupled-UCB term whenever
$
    \beta_T r_{\mathrm{causal}}\log T\to\infty,
$
which holds under standard choices of \(\beta_T\).
\end{corollary}

\begin{proof}
By \eqref{eq:param_error_rate},
$
    \epsilon_{\mathrm{param}}(N_t,\delta)
    =
    O_p(N_t^{-1/2}).
$
If
$
    N_t\asymp t^\alpha,
$
then
$
    N_t^{-1/2}\asymp t^{-\alpha/2}.
$
Therefore,
$
    \sum_{t=1}^T
    \epsilon_{\mathrm{param}}(N_t,\delta)
    =
    O_p
    \left(
        \sum_{t=1}^T t^{-\alpha/2}
    \right).
$
For \(0<\alpha\le 1\),
$
    \sum_{t=1}^T t^{-\alpha/2}
    =
    O(T^{1-\alpha/2}).
$
Substituting this rate into Theorem~\ref{thm:dynamic_regret} gives
\eqref{eq:growing_data_regret}. Setting \(\alpha=1\) gives
\eqref{eq:linear_data_regret}. The dominance statement follows by comparing
$
    \sqrt{
        T\beta_T r_{\mathrm{causal}}\log T
    }
$
with
$
    \sqrt T
$
as \(T\to\infty\).
\end{proof}

\subsection{Comparison with independent intervention-set GPs}
\label{subsec:comparison_independent_gps}

The CBO surrogate of \citet{aglietti2020causalbo} places a separate Gaussian
process prior on each intervention set \(s\in\mathcal E\), using do-calculus
to construct causal prior means and variance corrections. Thus, each
interventional response
$
    f_s(x_s)
    =
    \mathbb E[V_y\mid do(X_s=x_s)]
$
is modeled as a separate surrogate-modeling problem.

Let
$
    \gamma_T^{\mathrm{ind}}
$
denote the maximum information gain of the resulting independent multi-task
surrogate over \(T\) queries distributed across intervention sets. For
independent GPs, the joint kernel over the disjoint intervention-set domains is
block diagonal. Therefore, if \(T_s\) queries are allocated to intervention set
\(s\), with
$
    \sum_{s\in\mathcal E} T_s = T,
$
the information gain decomposes as
$
    \sum_{s\in\mathcal E}\gamma_{T_s}^{(s)}.
$
Consequently,
\begin{equation}
    \gamma_T^{\mathrm{ind}}
    =
    \max_{\{T_s\}_{s\in\mathcal E}:\ \sum_s T_s=T}
    \sum_{s\in\mathcal E}\gamma_{T_s}^{(s)}
    \le
    \sum_{s\in\mathcal E}\gamma_T^{(s)}.
    \label{eq:independent_info_gain}
\end{equation}
The final inequality is a worst-case upper bound, since each
\(\gamma_{T_s}^{(s)}\le \gamma_T^{(s)}\).

By contrast, the graph-coupled surrogate satisfies
\begin{equation}
    \gamma_T^{\mathrm{causal}}
    \le
    O(r_{\mathrm{causal}}\log T)
    \le
    O(\dim(\theta_Y)\log T).
    \label{eq:causal_info_gain_comparison}
\end{equation}
Thus graph coupling replaces a collection of independent information-gain
terms with a single finite-rank causal information-gain term.

Substituting these quantities into the GP-UCB regret form
$
    R_T
    =
    O\!\left(\sqrt{T\beta_T\gamma_T}\right),
$
the graph-coupled surrogate yields the optimization term
\begin{equation}
    R_T^{\mathrm{causal}}
    =
    O\!\left(
        \sqrt{
            T\beta_T r_{\mathrm{causal}}\log T
        }
    \right)
    \le
    O\!\left(
        \sqrt{
            T\beta_T \dim(\theta_Y)\log T
        }
    \right),
    \label{eq:causal_regret_comparison}
\end{equation}
whereas independent intervention-set GPs incur
\begin{equation}
    R_T^{\mathrm{ind}}
    =
    O\!\left(
        \sqrt{
            T\beta_T\gamma_T^{\mathrm{ind}}
        }
    \right).
    \label{eq:independent_regret_comparison}
\end{equation}
The advantage is largest when
$
    r_{\mathrm{causal}}\log T
    \ll
    \gamma_T^{\mathrm{ind}},
$
or, more conservatively, when
$
    \dim(\theta_Y)\log T
    \ll
    \gamma_T^{\mathrm{ind}}.
$
This is precisely the regime in which the causal graph is informative: multiple
interventional responses are different views of the same shared causal
mechanisms.

For example, in the chain graph of Example~\ref{ex:chain} with
$
    \mathcal E=\{\{X\},\{Z\}\},
$
the independent surrogate uses two separate one-dimensional GPs, one for
\(f_X\) and one for \(f_Z\). The graph-coupled surrogate instead represents
$
    f_X(x)=abx,
    \qquad
    f_Z(z)=bz
$
through the shared parameterization
$
    \theta_Y=(a,b).
$
If both \(a\) and \(b\) are unknown, then \(\dim(\theta_Y)=2\); if one
coefficient or total effect is known from prior data, the dimension decreases
accordingly. The independent surrogate cannot transfer information between
\(f_X\) and \(f_Z\), even though both depend on the shared parameter \(b\).
The graph-coupled surrogate exploits this overlap through the
cross-intervention covariance
$
    k_{\mathrm{causal}}\big((X,x),(Z,z)\big)
    =
    J_X(x)\Sigma_\theta J_Z(z)^\top.
$

We avoid claiming a universal dimensional comparison such as
$
    \sum_s \dim(X_s)
    \quad\text{versus}\quad
    \dim(\theta_Y)
$
without specifying the kernel class. Squared-exponential, Matérn, and
finite-dimensional linear kernels have distinct information-gain rates, so
ambient dimension alone is not the right comparison quantity. The clean
comparison is in terms of information gain: independent CBO accumulates
separate per-intervention-set information-gain terms, while graph-coupled CBO
uses a shared finite-rank causal information-gain term controlled by the
identifiable causal parameterization.

\subsection{Nonlinear extension: GP-per-mechanism and path-Lipschitz propagation}
\label{subsec:nonlinear_extension}

The linear Gaussian theory gives closed-form cross-intervention covariance and
finite-rank regret guarantees. For nonlinear SEMs, we keep the same
graph-factorized principle but do not claim the same finite-rank GP-UCB regret
theorem. Instead, we give a path-wise error propagation guarantee showing how
local mechanism-estimation errors affect interventional response estimates.

Assume each structural equation is
\begin{equation}
    V_i = g_i(\mathrm{Pa}_i) + \varepsilon_i,
    \qquad
    g_i \sim \mathcal{GP}.
    \label{eq:nonlinear_sem}
\end{equation}
For an intervention \(do(X_s=x_s)\), the graph is mutilated by replacing the
structural equations for \(X_s\) with constants \(x_s\), and uncertainty is
propagated through the remaining graph. This propagation may be implemented by
Monte Carlo sampling, moment matching, sigma-point propagation, or Bayesian
quadrature.

We assume each mechanism is coordinate-wise Lipschitz in its parents. That is,
for all parent configurations \(u,u'\),
\begin{equation}
    |g_j(u)-g_j(u')|
    \le
    \sum_{i\in \mathrm{Pa}(j)}
    L_{ij}|u_i-u_i'|.
    \label{eq:coordinate_lipschitz}
\end{equation}
This condition is stronger than ordinary joint Lipschitz continuity, but it is
the natural assumption for decomposing error propagation along individual
directed edges.

Let \(\widehat g_i\) be the learned mechanism on a compact domain
\(\mathcal U_i\), and suppose
\begin{equation}
    \sup_{u\in\mathcal U_i}
    |\widehat g_i(u)-g_i(u)|
    \le
    \epsilon_i.
    \label{eq:mechanism_error}
\end{equation}
For GP-learned mechanisms, \(\epsilon_i\) can be controlled in probability by
standard posterior contraction or GP regression concentration results on
compact domains. For example, one may have
$
    \epsilon_i = O_p(N_i^{-\beta_i})
$
for some rate \(\beta_i>0\) depending on the smoothness of \(g_i\), the kernel,
and the effective dimension of the parent domain.

For a directed path
$
    \pi:i\leadsto y,
$
define its path sensitivity as
\begin{equation}
    L(\pi)
    =
    \prod_{(a\to b)\in\pi}
    L_{ab}.
    \label{eq:path_sensitivity}
\end{equation}
Let \(G_{\overline X_s}\) denote the mutilated graph obtained by deleting
incoming edges into \(X_s\). The total path sensitivity from \(V_i\) to \(V_y\)
in the mutilated graph is
\begin{equation}
    L_{i\to y}^{(s)}
    =
    \sum_{\pi:i\leadsto y\text{ in }G_{\overline X_s}}
    L(\pi).
    \label{eq:total_path_sensitivity}
\end{equation}

\begin{proposition}[Path-Lipschitz nonlinear error propagation]
\label{prop:path_lipschitz}
Under \eqref{eq:coordinate_lipschitz} and \eqref{eq:mechanism_error}, the
error in the nonlinear intervention-response estimate satisfies
\begin{equation}
    |\widehat f_s(x_s)-f_s(x_s)|
    \le
    \sum_{i\in \mathrm{An}(y;G_{\overline X_s})}
    L_{i\to y}^{(s)}\epsilon_i.
    \label{eq:path_lipschitz_bound}
\end{equation}
\end{proposition}

\begin{proof}
Under \(do(X_s=x_s)\), incoming edges into \(X_s\) are removed. In the remaining
acyclic graph \(G_{\overline X_s}\), variables can be evaluated in topological
order.

For any non-intervened node \(j\),
\begin{align}
    |\widehat V_j - V_j|
    &=
    |\widehat g_j(\widehat{\mathrm{Pa}}_j)
      -
      g_j(\mathrm{Pa}_j)|
    \nonumber\\
    &\le
    |\widehat g_j(\widehat{\mathrm{Pa}}_j)
      -
      g_j(\widehat{\mathrm{Pa}}_j)|
    +
    |g_j(\widehat{\mathrm{Pa}}_j)
      -
      g_j(\mathrm{Pa}_j)|.
\end{align}
The first term is bounded by \(\epsilon_j\) by
\eqref{eq:mechanism_error}. The second term is bounded by the coordinate-wise
Lipschitz condition \eqref{eq:coordinate_lipschitz}. Therefore,
\begin{equation}
    |\widehat V_j-V_j|
    \le
    \epsilon_j
    +
    \sum_{i\in \mathrm{Pa}(j;G_{\overline X_s})}
    L_{ij}|\widehat V_i-V_i|.
    \label{eq:recursive_lipschitz_error}
\end{equation}
This inequality holds pointwise for every realization of the exogenous noise.
Unrolling \eqref{eq:recursive_lipschitz_error} along directed paths ending at
\(V_y\) gives
$
    |\widehat V_y-V_y|
    \le
    \sum_{i\in \mathrm{An}(y;G_{\overline X_s})}
    \left(
        \sum_{\pi:i\leadsto y\text{ in }G_{\overline X_s}}
        \prod_{(a\to b)\in\pi}L_{ab}
    \right)
    \epsilon_i.
$
The inner sum is exactly \(L_{i\to y}^{(s)}\). Hence,
$
    |\widehat V_y-V_y|
    \le
    \sum_{i\in \mathrm{An}(y;G_{\overline X_s})}
    L_{i\to y}^{(s)}\epsilon_i.
$
Since this bound holds pointwise, Jensen's inequality gives
\begin{align}
    |\widehat f_s(x_s)-f_s(x_s)|
    &=
    \left|
    \mathbb E_{do(X_s=x_s)}
    [\widehat V_y-V_y]
    \right|
    \nonumber\\
    &\le
    \mathbb E_{do(X_s=x_s)}
    \left[
    |\widehat V_y-V_y|
    \right]
    \nonumber\\
    &\le
    \sum_{i\in \mathrm{An}(y;G_{\overline X_s})}
    L_{i\to y}^{(s)}\epsilon_i.
\end{align}
This proves \eqref{eq:path_lipschitz_bound}.
\end{proof}

The bound shows that mechanism errors are amplified multiplicatively along
directed paths to \(V_y\): each edge contributes a factor \(L_{ab}\), and the
total contribution from ancestor \(V_i\) is the sum over all directed paths
from \(V_i\) to \(V_y\) in the mutilated graph. Graphs with long
high-sensitivity paths amplify uncertainty more than shallow graphs or graphs
with weak edge sensitivities.

\begin{remark}[Conservativeness of the path bound]
\label{rem:path_bound_conservative}
The bound \eqref{eq:path_lipschitz_bound} controls
$
    |\widehat f_s(x_s)-f_s(x_s)|
$
by applying Jensen's inequality to the absolute pointwise error. It is
therefore valid but may be conservative when pointwise mechanism errors
partially cancel under the intervened distribution.
\end{remark}

This proposition provides a graph-factorized predictive-accuracy guarantee for
the nonlinear extension of graph-coupled CBO. While the full regret theorem in
Section~\ref{subsec:regret} relies on the closed-form linear Gaussian
covariance, Proposition~\ref{prop:path_lipschitz} shows how local GP mechanism
errors can be combined with graph structure to control nonlinear
interventional-response error.

\subsection{Adaptive-kernel extension}
\label{subsec:adaptive_kernel_extension}

Assumption~\ref{ass:fixed_kernel} gives a clean fixed-kernel regret theorem,
but it is conservative in practice. As additional observational data arrive,
the posterior over the shared causal parameterization changes, so it is natural
to update
$
    \widehat\theta_{Y,t},
    \qquad
    \Sigma_{\theta,t},
    \qquad
    J_z(\widehat\theta_{Y,t}).
$
This leads to the adaptive causal kernel
\begin{equation}
    k_t(z,z')
    =
    J_z(\widehat\theta_{Y,t})
    \Sigma_{\theta,t}
    J_{z'}(\widehat\theta_{Y,t})^\top,
    \label{eq:adaptive_causal_kernel}
\end{equation}
where \(z=(s,x_s)\), \(z'=(t,x_t)\), and \(J_z\) denotes the Jacobian of the
corresponding interventional response with respect to the shared causal
parameterization.

The adaptive kernel is the empirically natural version of graph-coupled CBO:
as \(N_t\) grows, \(\Sigma_{\theta,t}\) typically shrinks and the surrogate
becomes less conservative. However, because the kernel changes over time,
standard GP-UCB regret theory does not apply directly. We therefore state a
conditional adaptive-kernel result with two explicit additional ingredients:
a kernel-drift confidence term and an adaptive variance-sum condition.

\begin{assumption}[Adaptive-kernel confidence with drift]
\label{ass:adaptive_kernel_drift}
At each round \(t\), the algorithm uses the adaptive kernel \(k_t\) in
\eqref{eq:adaptive_causal_kernel}. Assume that, with probability at least
\(1-\delta\), the corresponding posterior mean and standard deviation satisfy
\begin{equation}
    \left|
        f(z)-\widehat\mu_{t-1}^{(t)}(z)
    \right|
    \le
    \sqrt{\beta_t}\,
    \widehat\sigma_{t-1}^{(t)}(z)
    +
    \epsilon_t
    +
    \rho_t
    \qquad
    \forall z\in\mathcal Z,\quad t=1,\dots,T.
    \label{eq:adaptive_confidence_condition}
\end{equation}
Here \(\widehat\mu_{t-1}^{(t)}\) and
\(\widehat\sigma_{t-1}^{(t)}\) are the GP posterior mean and standard deviation
computed using the current kernel \(k_t\), \(\epsilon_t\) is the causal-estimation
error defined in \eqref{eq:epsilon_t_decomposition}, and \(\rho_t\ge 0\) is a
kernel-drift error accounting for the use of an estimated, time-varying kernel.
\end{assumption}

\begin{remark}[Interpretation of the drift term]
\label{rem:rho_interpretation}
The term \(\rho_t\) is an abstract high-probability bound on the error caused
by replacing a fixed reference kernel with the adaptive estimated kernel
\(k_t\). It absorbs changes in the linearization point, posterior covariance,
and induced posterior mean and variance. A typical perturbative control depends
on
$
    \|\widehat\theta_{Y,t}-\widehat\theta_{Y,t-1}\|,
    \qquad
    \|\Sigma_{\theta,t}-\Sigma_{\theta,t-1}\|,
    \qquad
    \sup_{z\in\mathcal Z}
    \|J_z(\widehat\theta_{Y,t})-J_z(\widehat\theta_{Y,t-1})\|.
$
For example, if the Jacobians are Lipschitz in \(\theta_Y\), the intervention
domain is compact, and the posterior covariance changes smoothly, then
\(\rho_t\) can be bounded by a perturbation term of the schematic form
$
    \rho_t
    \lesssim
    L_J
    \|\widehat\theta_{Y,t}-\widehat\theta_{Y,t-1}\|
    \|\Sigma_{\theta,t}\|
    +
    \|\Sigma_{\theta,t}-\Sigma_{\theta,t-1}\|,
$
up to constants depending on the bounded intervention domain. The adaptive
regret theorem below conditions on the event
\eqref{eq:adaptive_confidence_condition}; deriving sharp finite-sample bounds
on \(\rho_t\) for general adaptive kernels is a separate nonstationary-GP
problem.
\end{remark}

\begin{assumption}[Adaptive variance sum]
\label{ass:adaptive_variance_sum}
Let
$
    \overline\gamma_T
    =
    \sup_{1\le t\le T}
    \gamma_T(k_t)
$
be the worst-case maximum information gain over the adaptive kernel sequence.
Assume the adaptive posterior variances along the selected queries satisfy
\begin{equation}
    \sum_{t=1}^T
    \left(
        \widehat\sigma_{t-1}^{(t)}(z_t)
    \right)^2
    \le
    C_\gamma\,\overline\gamma_T
    \label{eq:adaptive_variance_sum}
\end{equation}
for a constant \(C_\gamma>0\).
\end{assumption}

\begin{remark}[Why Assumption~\ref{ass:adaptive_variance_sum} is needed]
For a fixed kernel, the standard GP-UCB variance-sum inequality bounds
$
    \sum_{t=1}^T
    \widehat\sigma_{t-1}^2(z_t)
$
by a constant multiple of the maximum information gain
\citep[Lemma~5.3]{srinivas2012information}. That argument relies on using the
same kernel throughout the sequential posterior update. When the kernel changes
with \(t\), the fixed-kernel proof no longer applies directly. Assumption
\ref{ass:adaptive_variance_sum} is therefore an explicit slow-drift-type
condition on the adaptive kernel sequence. It is expected to hold when
\(\{k_t\}\) changes sufficiently slowly, for example when
\(\|k_t-k_{t-1}\|_\infty\) decays with \(t\), but a general verification is left
for future work.
\end{remark}

\begin{theorem}[Approximate regret bound for adaptive graph-coupled CBO]
\label{thm:adaptive_kernel_regret}
Assume Assumptions~\ref{ass:identifiability}--\ref{ass:bounded_kernel},
Assumption~\ref{ass:adaptive_kernel_drift}, and
Assumption~\ref{ass:adaptive_variance_sum}. Suppose the adaptive causal-UCB
algorithm selects
\begin{equation}
    z_t
    \in
    \arg\max_{z\in\mathcal Z}
    \left\{
        \widehat\mu_{t-1}^{(t)}(z)
        +
        \sqrt{\beta_t}\,
        \widehat\sigma_{t-1}^{(t)}(z)
        +
        \epsilon_t
        -
        \eta c(z)
    \right\}.
    \label{eq:adaptive_ucb_rule}
\end{equation}
Then, on the event in
Assumption~\ref{ass:adaptive_kernel_drift}, the cumulative value regret
satisfies
\begin{equation}
    R_T
    \le
    C
    \sqrt{
        T\beta_T\overline\gamma_T
    }
    +
    2\sum_{t=1}^T
    \epsilon_{\mathrm{param}}(N_t,\delta)
    +
    2\sum_{t=1}^T
    \epsilon_{\mathrm{adj}}(N_t,M_t,\delta)
    +
    2\sum_{t=1}^T \rho_t
    +
    T\epsilon_{\mathcal E},
    \label{eq:adaptive_kernel_regret_bound}
\end{equation}
for a universal constant \(C>0\).

If each adaptive kernel \(k_t\) has rank at most \(r_{\mathrm{causal}}\) and
satisfies
$
    k_t(z,z)\le \kappa^2
    \qquad
    \forall z\in\mathcal Z,\quad t=1,\dots,T,
$
then
\begin{equation}
    \overline\gamma_T
    \le
    \frac{r_{\mathrm{causal}}}{2}
    \log
    \left(
        1+
        \frac{T\kappa^2}
             {\sigma^2 r_{\mathrm{causal}}}
    \right),
    \label{eq:adaptive_info_gain_bound}
\end{equation}
and therefore
\begin{equation}
    R_T
    \le
    C
    \sqrt{
        T\beta_T r_{\mathrm{causal}}\log T
    }
    +
    2\sum_{t=1}^T
    \epsilon_{\mathrm{param}}(N_t,\delta)
    +
    2\sum_{t=1}^T
    \epsilon_{\mathrm{adj}}(N_t,M_t,\delta)
    +
    2\sum_{t=1}^T \rho_t
    +
    T\epsilon_{\mathcal E}.
    \label{eq:adaptive_kernel_rank_regret_bound}
\end{equation}
\end{theorem}

\begin{proof}
Let
$
    z_t^\star
    \in
    \arg\max_{z\in\mathcal Z} f(z).
$
On the event in Assumption~\ref{ass:adaptive_kernel_drift}, we have the
upper and lower confidence inequalities
$
    f(z)
    \le
    \widehat\mu_{t-1}^{(t)}(z)
    +
    \sqrt{\beta_t}
    \widehat\sigma_{t-1}^{(t)}(z)
    +
    \epsilon_t
    +
    \rho_t
$
and
$
    f(z)
    \ge
    \widehat\mu_{t-1}^{(t)}(z)
    -
    \sqrt{\beta_t}
    \widehat\sigma_{t-1}^{(t)}(z)
    -
    \epsilon_t
    -
    \rho_t
$
for all \(z\in\mathcal Z\). By the same upper-confidence/lower-confidence
argument used in the proof of Theorem~\ref{thm:dynamic_regret}, now with the
additional drift term \(\rho_t\), the instantaneous regret within
\(\mathcal E\) satisfies
\begin{equation}
    f(z_t^\star)-f(z_t)
    \le
    2\sqrt{\beta_t}
    \widehat\sigma_{t-1}^{(t)}(z_t)
    +
    2\epsilon_t
    +
    2\rho_t,
    \label{eq:adaptive_instant_regret}
\end{equation}
for value regret, taking \(\eta=0\) or treating the cost-adjusted objective
separately as in the fixed-kernel theorem.

Summing \eqref{eq:adaptive_instant_regret} over \(t\) gives
$
    \sum_{t=1}^T
    \left(
        f^\star_{\mathcal E}-f(z_t)
    \right)
    \le
    2\sum_{t=1}^T
    \sqrt{\beta_t}\,
    \widehat\sigma_{t-1}^{(t)}(z_t)
    +
    2\sum_{t=1}^T \epsilon_t
    +
    2\sum_{t=1}^T \rho_t.
$
Assuming \(\beta_t\le \beta_T\) and applying Cauchy--Schwarz,
$
    \sum_{t=1}^T
    \sqrt{\beta_t}\,
    \widehat\sigma_{t-1}^{(t)}(z_t)
    \le
    \sqrt{
        T\beta_T
        \sum_{t=1}^T
        \left(
            \widehat\sigma_{t-1}^{(t)}(z_t)
        \right)^2
    }.
$
By Assumption~\ref{ass:adaptive_variance_sum},
$
    \sum_{t=1}^T
    \left(
        \widehat\sigma_{t-1}^{(t)}(z_t)
    \right)^2
    \le
    C_\gamma\overline\gamma_T.
$
Therefore,
$
    \sum_{t=1}^T
    \sqrt{\beta_t}\,
    \widehat\sigma_{t-1}^{(t)}(z_t)
    \le
    C'
    \sqrt{
        T\beta_T\overline\gamma_T
    }.
$
Adding the exploration-set approximation term \(T\epsilon_{\mathcal E}\) and
using
$
    \epsilon_t
    =
    \epsilon_{\mathrm{param}}(N_t,\delta)
    +
    \epsilon_{\mathrm{adj}}(N_t,M_t,\delta)
$
gives \eqref{eq:adaptive_kernel_regret_bound}.

It remains to bound \(\overline\gamma_T\). For each fixed \(t\), the kernel
\(k_t\) is finite-rank with rank at most \(r_{\mathrm{causal}}\) and diagonal
bounded by \(\kappa^2\). Applying Lemma~\ref{lem:finite_rank_information_gain}
to \(k_t\) gives
$
    \gamma_T(k_t)
    \le
    \frac{r_{\mathrm{causal}}}{2}
    \log
    \left(
        1+
        \frac{T\kappa^2}
             {\sigma^2 r_{\mathrm{causal}}}
    \right).
$
Taking the supremum over \(t=1,\dots,T\) gives
\eqref{eq:adaptive_info_gain_bound}. Substituting this bound into
\eqref{eq:adaptive_kernel_regret_bound} yields
\eqref{eq:adaptive_kernel_rank_regret_bound}.
\end{proof}

\begin{remark}[Comparison with the fixed-reference theorem]
Compared with the fixed-kernel bound in Theorem~\ref{thm:dynamic_regret}, the
adaptive-kernel bound adds the cumulative kernel-drift penalty
$
    2\sum_{t=1}^T \rho_t
$
and requires the adaptive variance-sum condition in
Assumption~\ref{ass:adaptive_variance_sum}. Under slow drift conditions where
\(\rho_t=O(t^{-1/2})\), the drift contribution is \(O(\sqrt T)\), matching the
causal-estimation term under linear observational growth and not dominating the
coupled-UCB term up to logarithmic and rank factors. Thus the adaptive theorem
captures the practical benefit of updating the causal kernel while making
explicit the extra nonstationary-kernel assumptions needed for regret control.
\end{remark}

\subsection{Summary of theoretical improvement}
\label{subsec:theory_summary}

The theoretical spine of graph-coupled CBO in the identifiable linear Gaussian
case is:
\[
    \text{identifiable shared SEM parameters}
    \quad
    (\text{Assumption~\ref{ass:identifiability}})
\]
\[
    \Downarrow
\]
\[
    \operatorname{Cov}
    \left(
        f_s(x_s),f_t(x_t)\mid D^O
    \right)
    =
    J_s(x_s)\Sigma_\theta J_t(x_t)^\top
    \quad
    (\text{Theorem~\ref{thm:cross_intervention_covariance}})
\]
\[
    \Downarrow
\]
\[
    \operatorname{rank}(K_T)
    \le
    r_{\mathrm{causal}}
    \le
    \dim(\theta_Y)
    \quad
    (\text{Corollary~\ref{cor:finite_causal_rank}})
\]
\[
    \Downarrow
\]
\[
    \gamma_T^{\mathrm{causal}}
    =
    O(r_{\mathrm{causal}}\log T)
    \quad
    (\text{Lemma~\ref{lem:finite_rank_information_gain}})
\]
\[
    \Downarrow
\]
\[
    R_T
    \le
    C
    \sqrt{
        T\beta_T r_{\mathrm{causal}}\log T
    }
    +
    2\sum_{t=1}^T
    \epsilon_{\mathrm{param}}(N_t,\delta)
    +
    2\sum_{t=1}^T
    \epsilon_{\mathrm{adj}}(N_t,M_t,\delta)
    +
    T\epsilon_{\mathcal E}
    \quad
    (\text{Theorem~\ref{thm:dynamic_regret}}).
\]

The causal graph produces an optimization advantage by coupling
intervention-response functions through shared identifiable mechanisms. The
regret depends on the rank of the shared causal parameterization rather than
on a collection of independent intervention-set surrogates
(Section~\ref{subsec:comparison_independent_gps}).

For nonlinear SEMs, Proposition~\ref{prop:path_lipschitz} provides a
graph-factorized path-Lipschitz error guarantee in place of the closed-form
linear Gaussian covariance. For the practical adaptive-kernel implementation,
Theorem~\ref{thm:adaptive_kernel_regret} gives an approximate regret bound
under explicit kernel-drift and adaptive variance-sum conditions.

\section{Algorithm}
\label{sec:algorithm}

This section translates the graph-coupled theory into a concrete sequential
optimization procedure. The algorithm maintains three objects across rounds:
(i) an observational dataset \(D^O_{N_t}\), which determines the posterior over
the shared causal parameterization \(\theta_Y\); (ii) an interventional dataset
\(D^I_t=\{(z_i,y_i)\}_{i=1}^{t-1}\), where \(z_i=(s_i,x_{s_i})\); and
(iii) a coupled causal Gaussian process surrogate whose kernel is induced by
the posterior uncertainty in \(\theta_Y\).

\begin{algorithm}[!ht]
\DontPrintSemicolon
\SetKwInOut{Input}{Input}
\SetKwInOut{Output}{Output}
\SetKwComment{Comment}{// }{}

\Input{Causal graph \(G\); exploration set \(\mathcal E\);
initial observational data \(D^O_{N_0}\);
intervention domains \(\{\mathcal X_s\}_{s\in\mathcal E}\);
horizon \(T\);
cost function \(c\); cost weight \(\eta\ge 0\);
confidence level \(\delta\in(0,1)\);
observation budget \(N_{\max}\);
observation probability \(p_{\mathrm{obs}}\).}

\Output{Estimated best intervention
\(\widehat z^\star=(\widehat s^\star,\widehat x^\star)\)
and value \(\widehat f^\star\).}

\BlankLine
\textbf{Initialization:}\;
Construct the shared parameterization \(\theta_Y\) by applying the ID
algorithm to each query \(P(V_y\mid do(X_s=x_s))\), \(s\in\mathcal E\)\;
Compute the initial posterior
\((\widehat\theta_{Y,0},\Sigma_{\theta,0})\) from \(D^O_{N_0}\)\;
Compute the response maps \(g_s(x_s;\theta_Y)\) and Jacobians
\(J_s(x_s;\widehat\theta_{Y,0})\) for \(s\in\mathcal E\)\;
Set \(D^I_0\leftarrow\emptyset\) and \(N\leftarrow N_0\)\;

\BlankLine
\For{\(t=1,\dots,T\)}{
    Compute $\epsilon_t
        =
        \epsilon_{\mathrm{param}}(N,\delta)
        +
        \epsilon_{\mathrm{adj}}(N,M_t,\delta)$\;
    Sample \(u\sim\mathrm{Uniform}(0,1)\)\;

    \uIf{\(u<p_{\mathrm{obs}}\) \textbf{and} \(N<N_{\max}\)}{
        \Comment*[h]{Observe step}\;
        Collect observational sample \(v_t\sim P(V)\)\;
        Set \(N\leftarrow N+1\)\;
        \(D^O_N\leftarrow D^O_{N-1}\cup\{v_t\}\)\;
        Update \((\widehat\theta_{Y,t},\Sigma_{\theta,t})\) from
        \(D^O_N\)\;
        Set \(D^I_t\leftarrow D^I_{t-1}\)\;
    }
    \Else{
        \Comment*[h]{Intervene step}\;
        Build the coupled causal kernel
        $k_t(z,z')
            =
            J_z(\widehat\theta_{Y,t})
            \Sigma_{\theta,t}
            J_{z'}(\widehat\theta_{Y,t})^\top$\;
            
        for \(z,z'\in\mathcal Z\)
        \Comment*[h]{or use \(k_0\) in the fixed-reference variant}\;
        Compute coupled GP posterior mean \(\widehat\mu_{t-1}(z)\) and
        standard deviation \(\widehat\sigma^{\mathrm{func}}_{t-1}(z)\)
        using \(D^I_{t-1}\)\;
        Select $z_t
            \in
            \arg\max_{z\in\mathcal Z}
            \left\{
                \widehat\mu_{t-1}(z)
                +
                \sqrt{\beta_t}\,
                \widehat\sigma^{\mathrm{func}}_{t-1}(z)
                +
                \epsilon_t
                -
                \eta c(z)
            \right\}$\;
        Perform \(do(X_{s_t}=x_{s_t})\), where \(z_t=(s_t,x_{s_t})\)\;
        Observe $y_t=f_{s_t}(x_{s_t})+\xi_t,
            \quad
            \xi_t\sim\mathcal N(0,\sigma^2)$\;
        \(D^I_t\leftarrow D^I_{t-1}\cup\{(z_t,y_t)\}\)\;
        Update the coupled GP posterior\;
    }
}

\BlankLine
Return $\widehat z^\star
    \in
    \arg\max_{(z_i,y_i)\in D^I_T}
    \widehat\mu_T(z_i),
    \quad
    \widehat f^\star=\widehat\mu_T(\widehat z^\star).$\;

\caption{Graph-coupled causal Bayesian optimization (GC-CBO)}
\label{alg:gccbo}
\end{algorithm}

\subsection{Graph-coupled CBO}
\label{subsec:algorithm_description}

Let
$
    \mathcal Z
    =
    \{(s,x_s):s\in\mathcal E,\ x_s\in\mathcal X_s\}
$
denote the joint intervention domain. In the fixed-reference version analyzed
in Theorem~\ref{thm:dynamic_regret}, the causal kernel is constructed once
from
$
    (\widehat\theta_{Y,0},\Sigma_{\theta,0}).
$
In the adaptive version, the kernel is updated as \(D^O_{N_t}\) grows, as
described in Section~\ref{subsec:adaptive_kernel_extension}. The pseudocode
below is written for the adaptive implementation; the fixed-reference variant
is obtained by replacing
$
    (\widehat\theta_{Y,t},\Sigma_{\theta,t})
$
by
$
    (\widehat\theta_{Y,0},\Sigma_{\theta,0})
$
in the kernel construction.

\subsection{Explanation and connection to the theory}
\label{subsec:algorithm_explanation}

\noindent\textbf{Constructing \(\theta_Y\).}
The initialization step applies the ID algorithm to every interventional query
in the exploration set and collects the unique identifiable causal quantities
needed to express all response functions
$
    f_s(x_s)=g_s(x_s;\theta_Y).
$
This is the constructive content of Assumption~\ref{ass:identifiability}.
The resulting dimension \(\dim(\theta_Y)\) is the quantity that controls the
finite-rank information-gain bound.

\noindent\textbf{Posterior over \(\theta_Y\).}
The posterior
$
    (\widehat\theta_{Y,t},\Sigma_{\theta,t})
$
is estimated from observational data. In conjugate linear Gaussian models this
update is closed-form; for smooth identifiable functionals,
\(\Sigma_{\theta,t}\) is interpreted as the posterior covariance of the chosen
shared parameterization, as in Assumption~\ref{ass:posterior_regular}.

\noindent\textbf{Coupled causal kernel.}
At an intervention query \(z=(s,x_s)\), the Jacobian
$
    J_z=J_s(x_s)
$
measures the sensitivity of the interventional response to the shared causal
parameters. The kernel
$
    k_t(z,z')
    =
    J_z\Sigma_{\theta,t}J_{z'}^\top
$
therefore couples intervention sets through posterior uncertainty in common
causal mechanisms. This is the algorithmic realization of
Theorem~\ref{thm:cross_intervention_covariance} and
Corollary~\ref{cor:finite_causal_rank}.

\noindent\textbf{Observation versus intervention.}
The observation branch reduces causal-estimation error by shrinking
\(\Sigma_{\theta,t}\) and therefore
$
    \epsilon_{\mathrm{param}}(N_t,\delta).
$
The intervention branch reduces optimization error by sampling high-value or
high-uncertainty interventional queries. The Bernoulli observation schedule is
a simple baseline policy. More adaptive policies can choose whether to observe
or intervene based on the current size of \(\Sigma_{\theta,t}\), the remaining
budget, or the expected reduction in \(\epsilon_t\).

\noindent\textbf{Regret-safe acquisition.}
The intervention branch uses
$
    \alpha_t^{\mathrm{safe}}(z)
    =
    \widehat\mu_{t-1}(z)
    +
    \sqrt{\beta_t}\,
    \widehat\sigma^{\mathrm{func}}_{t-1}(z)
    +
    \epsilon_t
    -
    \eta c(z).
$
The first two terms are the standard GP-UCB mean and exploration bonus. The
additive \(\epsilon_t\) term is the high-probability correction for causal
estimation error, introduced in Section~\ref{subsec:causal_ucb}. The cost
term penalizes expensive interventions. Setting \(\eta=0\) recovers the pure
value-regret objective analyzed in Theorem~\ref{thm:dynamic_regret}.

\noindent\textbf{Reporting rule.}
The algorithm returns the best intervention among those evaluated, measured by
the final posterior mean. This is the standard reporting rule for sequential
black-box optimization. One may instead return the best observed interventional
outcome or the posterior maximizer over all of \(\mathcal Z\).

\subsection{Soundness}
\label{subsec:algorithm_soundness}

The following proposition states that the algorithm inherits the regret
guarantee of the theory section when run in the fixed-reference setting.

\begin{proposition}[Soundness of fixed-reference GC-CBO]
\label{prop:gccbo_soundness}
Suppose Assumptions~\ref{ass:identifiability}--\ref{ass:fixed_kernel} hold.
Run Algorithm~\ref{alg:gccbo} in its fixed-reference variant with \(\eta=0\),
and suppose the causal-estimation event
\eqref{eq:causal_error_bound_theorem} and the standard GP-UCB confidence event
hold. Then the sequence of intervention queries satisfies the regret bound in
\eqref{eq:dynamic_regret_dim_bound}. If, in addition,
$
    \epsilon_{\mathrm{adj}}(N_t,M_t,\delta)=0,
    \qquad
    \epsilon_{\mathcal E}=0,
$
and
$
    N_t\asymp t,
$
then
\begin{equation}
    \frac{R_T}{T}
    =
    O\!\left(
        \sqrt{
            \frac{\beta_T\dim(\theta_Y)\log T}{T}
        }
    \right)
    +
    O_p(T^{-1/2}).
    \label{eq:gccbo_average_regret}
\end{equation}
Consequently, the best true intervention encountered by the algorithm
satisfies the same rate in simple regret:
\begin{equation}
    f^\star_{\mathrm{all}}
    -
    \max_{1\le t\le T} f(z_t)
    \le
    \frac{R_T}{T}.
    \label{eq:simple_regret_from_average}
\end{equation}
\end{proposition}

\begin{proof}
In the fixed-reference variant, Algorithm~\ref{alg:gccbo} uses the kernel
$
    k_{\mathrm{causal}}(z,z')
    =
    J_z(\widehat\theta_{Y,0})
    \Sigma_{\theta,0}
    J_{z'}(\widehat\theta_{Y,0})^\top
$
and the regret-safe acquisition analyzed in
Theorem~\ref{thm:dynamic_regret}. Therefore the regret bound in \eqref{eq:dynamic_regret_dim_bound} follows directly from
Theorem~\ref{thm:dynamic_regret} and
Corollary~\ref{cor:finite_causal_rank}. Under linear observational growth,
Corollary~\ref{cor:growing_observational_data} gives
\eqref{eq:gccbo_average_regret}. Finally,
$
    f^\star_{\mathrm{all}}
    -
    \max_{1\le t\le T} f(z_t)
    \le
    \frac{1}{T}
    \sum_{t=1}^T
    \left(
        f^\star_{\mathrm{all}}-f(z_t)
    \right)
    =
    \frac{R_T}{T},
$
which proves the simple-regret claim.
\end{proof}

\begin{remark}[Returned posterior maximizer]
The simple-regret guarantee in Proposition~\ref{prop:gccbo_soundness} applies
to the best true queried intervention,
$
    \arg\max_{1\le t\le T} f(z_t).
$
The posterior-mean reporting rule in Algorithm~\ref{alg:gccbo} is a practical
estimator of this best queried intervention and is consistent on the same
confidence event when posterior mean error vanishes.
\end{remark}

\begin{remark}[Adaptive implementation]
If Algorithm~\ref{alg:gccbo} updates the kernel at every round using
$
    (\widehat\theta_{Y,t},\Sigma_{\theta,t}),
$
then its guarantee is the adaptive-kernel bound in
Theorem~\ref{thm:adaptive_kernel_regret}, provided the drift and adaptive
variance-sum assumptions of Section~\ref{subsec:adaptive_kernel_extension}
hold.
\end{remark}

\subsection{Complexity analysis}
\label{subsec:algorithm_complexity}

Let
$
    d=\dim(\theta_Y),
    \qquad
    r=r_{\mathrm{causal}}\le d,
$
and let \(M=|\mathcal Z|\) denote the number of candidate intervention queries
after discretizing the domains \(\mathcal X_s\). For continuous domains,
\(M\) should be replaced by the cost of the numerical acquisition optimizer.

\noindent\textbf{Posterior update over \(\theta_Y\).}
A direct Bayesian linear regression update over \(d\) parameters costs
\(O(d^3)\) if a full covariance matrix is inverted. With rank-one or online
updates, this can be reduced to \(O(d^2)\) per observational sample after an
initial factorization.

\noindent\textbf{Coupled GP update.}
After \(t\) interventional evaluations, the naive GP posterior update costs
\(O(t^3)\) time and \(O(t^2)\) memory. Because the causal kernel has the
low-rank form
$
    K_t=\Phi_t\Phi_t^\top,
    \qquad
    \Phi_t\in\mathbb R^{t\times r},
$
the matrix inversion lemma gives
$
    O(tr^2+r^3)
$
time for the core low-rank solve and
$
    O(tr+r^2)
$
memory, assuming \(\Phi_t\) has been computed.

\noindent\textbf{Acquisition evaluation.}
Evaluating the acquisition over \(M\) candidate queries requires computing
feature vectors
$
    \phi(z)=\Sigma_{\theta,t}^{1/2}J_z^\top
$
and posterior mean/variance terms. Direct evaluation costs \(O(Mr)\) for
feature construction when Jacobians are available, plus posterior variance
evaluation. Using the low-rank posterior representation, this is typically
$
    O(Mr^2)
$
per acquisition sweep. For continuous domains, \(M\) is replaced by the number
of optimizer evaluations times the same per-query cost.

\noindent\textbf{Total complexity.}
Ignoring observational-data storage and assuming \(T\) intervention rounds, the
low-rank coupled GP computations cost approximately
$
    O(T^2r^2+Tr^3)
$
over the full run, while exhaustive acquisition evaluation over a discretized
domain contributes
$
    O(TMr^2).
$
The total computational cost is therefore
\begin{equation}
    O\!\left(
        T^2r^2
        +
        Tr^3
        +
        TMr^2
    \right),
    \label{eq:gccbo_complexity}
\end{equation}
plus the cost of updating the observational posterior over \(\theta_Y\). Since
\(r\le d\), this is substantially cheaper than a full GP update when
\(r\ll T\).

\noindent\textbf{Memory complexity.}
The low-rank implementation stores the feature matrix
$
    \Phi_t\in\mathbb R^{t\times r}
$
and the posterior covariance over \(\theta_Y\), requiring
$
    O(tr+d^2)
$
memory plus observational data storage. This improves over the \(O(t^2)\)
memory required by storing a dense GP kernel matrix.

\noindent\textbf{Comparison with independent CBO.}
Independent intervention-set CBO maintains separate GP surrogates for each
\(s\in\mathcal E\). If \(t_s\) interventions are allocated to set \(s\), then
dense kernel factorizations recomputed from scratch have total cost
$
    \sum_{s\in\mathcal E} O(t_s^3)
$
across the separate kernels. With incremental Cholesky updates this can be
reduced, but the independent surrogates still do not share information across
intervention sets. Graph-coupled CBO instead uses one low-rank surrogate over
the joint domain. The main advantage is statistical rather than only
computational: interventional data from one set can reduce posterior
uncertainty for another whenever the responses share identifiable causal
parameters, as quantified by
Section~\ref{subsec:comparison_independent_gps}.

\section{Experimental Results}
\label{sec:experiments}

We evaluate GC-CBO from the strongest theory-aligned setting to progressively more realistic and more challenging scenarios.  All experiments are minimization problems and report the best value found so far as a function of cumulative intervention cost.  For each method, curves show the median trajectory across random seeds with interquartile bands.  Tables report terminal mean, terminal standard deviation, terminal median, cumulative cost, and the first cost at which the median curve reaches an $\epsilon$-neighborhood of the optimum.  Unless otherwise stated, GC-CBO uses the paper-aligned lower-confidence-bound acquisition
\[
        \mu_t(z)-\sqrt{\beta_t}\sigma_t(z)+\eta c(z),
\]
where the posterior covariance includes the proposed causal finite-rank coupling term.  The extra causal-uncertainty bonus is set to zero in the main experiments to avoid double-counting uncertainty already present in the coupled posterior.

In all experiments, the horizontal axis is cumulative \emph{intervention} cost rather than the number of optimization iterations.  Observational data are used as an offline resource to estimate the initial causal model, causal prior, and parameter uncertainty; they are not charged on the intervention-cost axis.  Each interventional query incurs cost equal to the number of variables actively manipulated in that query.  Thus, an intervention on a singleton costs $1$, an intervention on a pair costs $2$, and an intervention on a set $S$ costs $|S|$.  This convention reflects the increasing experimental burden of manipulating more variables and is especially important in the ECOLI70 no-parent study, where allowing larger intervention sets both expands the feasible search family and permits more expensive high-order interventions.

\noindent\textbf{Linear-Gaussian theory validation.}
We first test the setting that directly matches the theory: a linear-Gaussian chain
\[
        X \longrightarrow Z \longrightarrow Y,
        \qquad
        Z=aX+\varepsilon_Z,
        \qquad
        Y=bZ+\varepsilon_Y .
\]
This example is deliberately small, but it is the cleanest test of the main mechanism.  The two intervention functions,
\[
        f_X(x)=\mathbb{E}[Y\mid do(X=x)]
        \quad\text{and}\quad
        f_Z(z)=\mathbb{E}[Y\mid do(Z=z)],
\]
are not independent black-box functions: both are determined by the same low-dimensional structural parameter vector $\theta_Y=(a,b)$.  GC-CBO should therefore induce a low-rank cross-intervention covariance, whereas independent CBO treats the two intervention functions separately.  The corresponding finite-rank, cross-covariance, information-gain, and parameter-concentration diagnostics are shown in Figure~\ref{fig:theory-diagnostics}, with the main numerical quantities summarized in Table~\ref{tab:theory-summary}.

The fitted parameters were
\[
        \widehat{\theta}_Y=(0.8515,-1.3654),
        \qquad
        \theta_Y=(0.8,-1.3).
\]
As reported in Table~\ref{tab:theory-summary} and visualized in the first panel of Figure~\ref{fig:theory-diagnostics}, the empirical coupled kernel had
\[
        \operatorname{rank}(K_T)=2=\dim(\theta_Y),
\]
which exactly matches the finite-rank prediction.  Moreover, the numerically estimated cross-intervention covariance agreed with the analytic covariance to machine precision over the tested query pairs, as shown in the second panel of Figure~\ref{fig:theory-diagnostics} and summarized by the zero cross-covariance error in Table~\ref{tab:theory-summary}.  This is the most direct validation that the implementation is using the intended causal kernel rather than merely adding another generic covariance function.

\begin{figure*}[t]
    \centering
    \includegraphics[width=.24\textwidth]{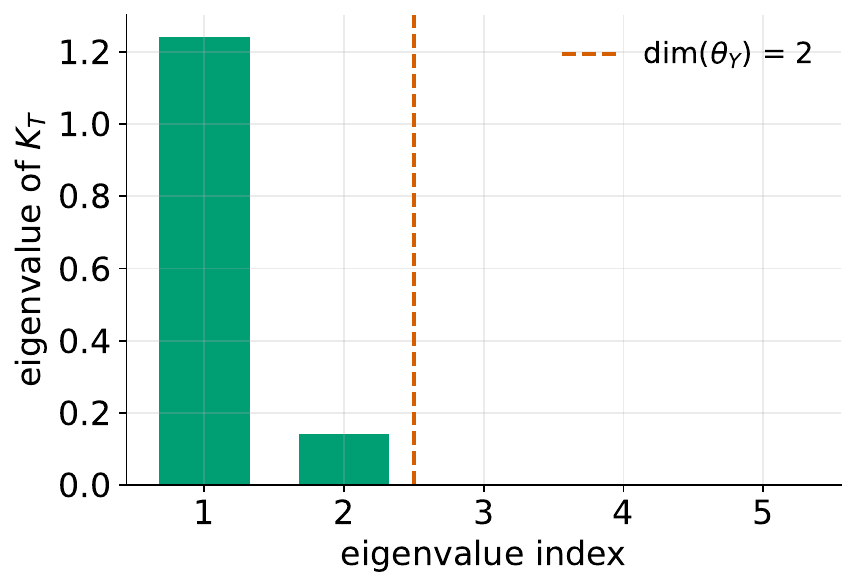}
    \includegraphics[width=.24\textwidth]{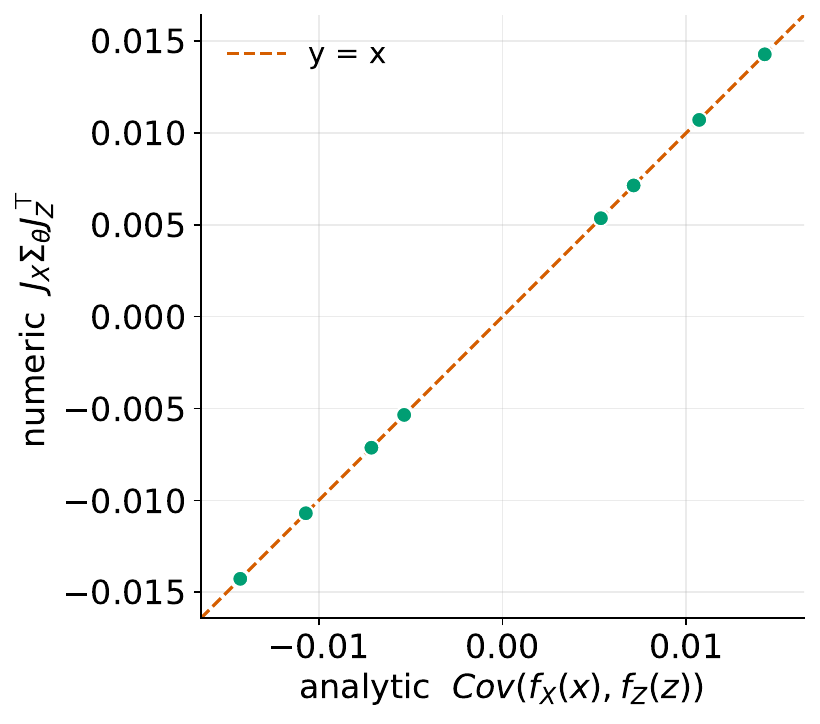}
    \includegraphics[width=.24\textwidth]{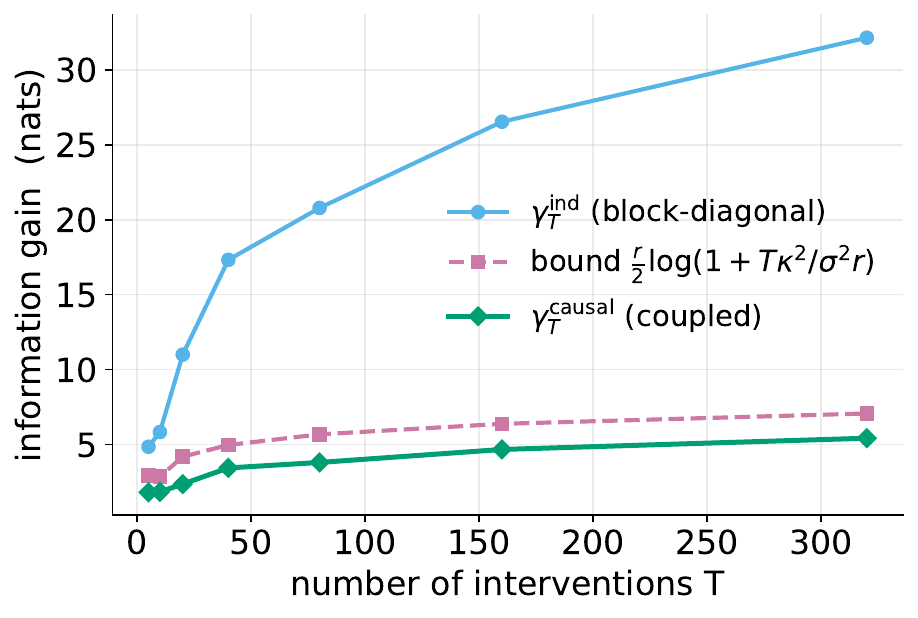}
    \includegraphics[width=.24\textwidth]{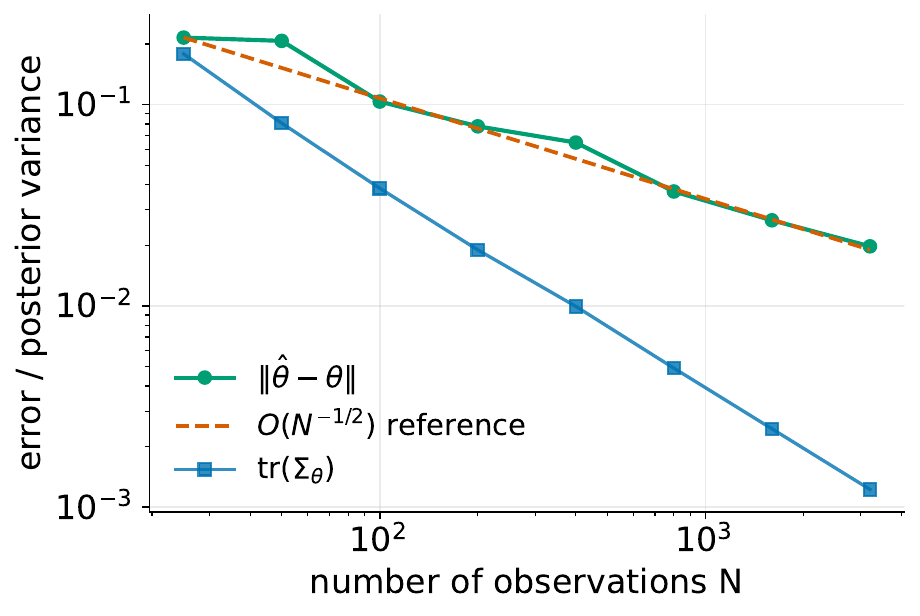}
    \caption{
    Theory-aligned diagnostics on the linear-Gaussian chain.
    Left to right: finite rank of the causal kernel, analytic versus numerical cross-intervention covariance, information gain of the causal kernel compared with an independent kernel, and concentration of the fitted structural parameters as the observational sample size increases.
    }
    \label{fig:theory-diagnostics}
\end{figure*}

The information-gain experiment further supports the theory, as shown in the third panel of Figure~\ref{fig:theory-diagnostics} and reported in Table~\ref{tab:theory-summary}.  At $T=320$, the causal kernel has information gain $5.416$, while the corresponding independent kernel has information gain $32.152$.  Thus, the coupled kernel grows far more slowly because it reuses a shared causal parameterization across intervention functions.  The empirical value also remains below the finite-rank bound used in the regret analysis, as shown in Table~\ref{tab:theory-summary}.  Finally, the parameter-concentration diagnostic in the fourth panel of Figure~\ref{fig:theory-diagnostics} shows that the estimation error decreases as the observational sample size grows; the trace of the parameter covariance falls from $0.1784$ at $N=25$ to $0.00123$ at $N=3200$, with the latter value reported in Table~\ref{tab:theory-summary}.  These results justify the central modeling choice of GC-CBO: when intervention functions are linked by shared structural mechanisms, the optimizer should exploit this coupling rather than learning each function from scratch.

\begin{table}[t]
\centering
\caption{Theory-aligned diagnostics on the linear-Gaussian chain.  The coupled causal kernel has rank equal to the target-relevant parameter dimension and substantially smaller information gain than an independent kernel.}
\label{tab:theory-summary}
\begin{tabular}{lccc}
\toprule
Diagnostic & Quantity & Value & Interpretation \\
\midrule
Finite rank & $\operatorname{rank}(K_T)$ & $2$ & Equals $\dim(\theta_Y)=2$ \\
Cross-covariance error & $\max |\widehat{k}-k|$ & $0.0$ & Matches analytic covariance \\
Information gain & $\gamma_{\rm causal}(320)$ & $5.416$ & Below rank-based bound $7.068$ \\
Independent baseline & $\gamma_{\rm ind}(320)$ & $32.152$ & Much larger than causal kernel \\
Parameter concentration & $\operatorname{tr}(\Sigma_{\theta})$ at $N=3200$ & $0.00123$ & Decreases with observational data \\
\bottomrule
\end{tabular}
\end{table}

\noindent\textbf{Cross-set transfer stress tests.}
The finite-rank diagnostics in Figure~\ref{fig:theory-diagnostics} and Table~\ref{tab:theory-summary} establish that the kernel has the right algebraic form.  We next ask whether this structure improves optimization when the intervention family is large and only sparse interventional data are available.  To isolate this effect, we use two stress tests in which several intervention functions share a downstream mechanism.  In the linear confounded-funnel case, held-out intervention sets are predicted from a small number of observed intervention outcomes on related sets.  Independent CBO cannot transfer across sets, while GC-CBO can transfer through the shared causal kernel.  The resulting transfer behavior is plotted in Figure~\ref{fig:stress-tests} and quantified in Table~\ref{tab:stress-results}.

The difference is substantial, as shown by the left panel of Figure~\ref{fig:stress-tests} and the linear-funnel columns of Table~\ref{tab:stress-results}.  With budget $3$, GC-CBO has prediction error $0.7707$, compared with $1.5801$ for independent CBO.  As the budget increases to $48$, GC-CBO improves to $0.1726$, while CBO remains around $1.6320$.  This is the clearest empirical demonstration of the proposed mechanism: GC-CBO does not merely optimize a single intervention function; it learns shared causal structure that generalizes across related intervention functions.

\begin{figure*}[t]
    \centering
    \includegraphics[width=.47\textwidth]{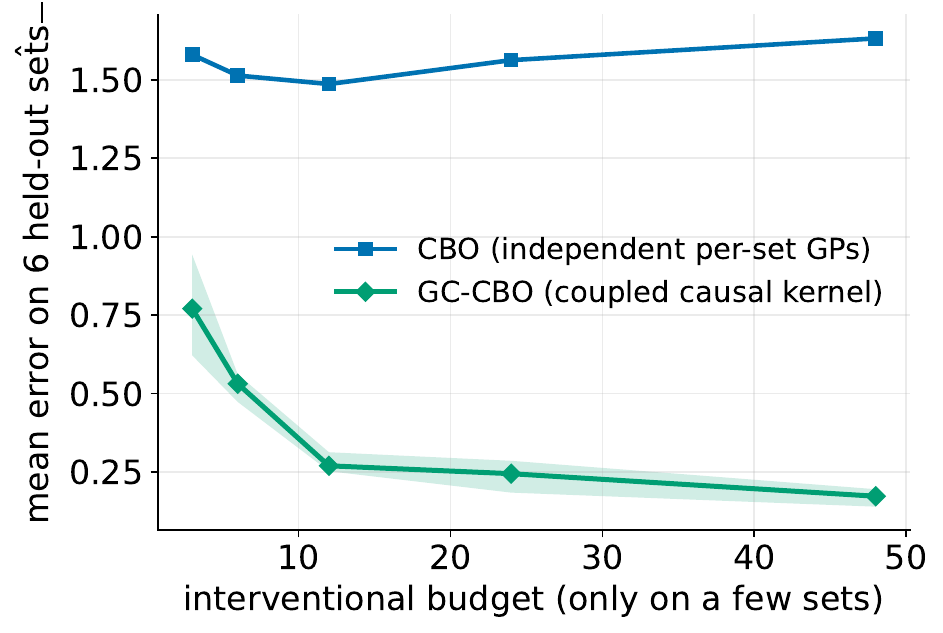}
    \includegraphics[width=.47\textwidth]{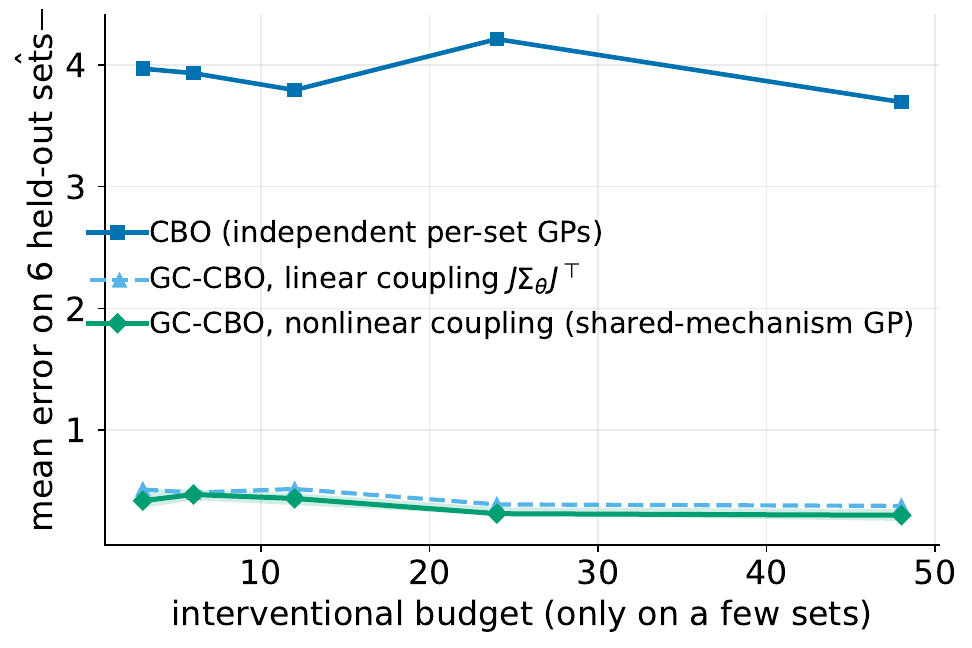}
    \caption{
    Cross-set transfer stress tests.  Left: linear confounded-funnel setting, where GC-CBO substantially reduces held-out intervention error compared with independent CBO.  Right: nonlinear shared-mechanism setting, where both linearized and nonlinear GC-CBO variants outperform independent CBO, with the nonlinear variant giving the best performance.
    }
    \label{fig:stress-tests}
\end{figure*}

The nonlinear shared-mechanism stress test gives a similar message but also reveals a limitation, as shown in the right panel of Figure~\ref{fig:stress-tests} and the nonlinear-funnel columns of Table~\ref{tab:stress-results}.  Independent CBO remains inaccurate, with error near $3.69$--$4.21$ across budgets.  GC-CBO with a linearized causal kernel reduces the error to about $0.38$ by budget $48$, while the nonlinear shared-mechanism version reaches $0.3017$.  Thus, even an approximate causal coupling is useful, but nonlinear mechanisms benefit from richer causal features.  This supports the paper's emphasis on finite-rank and locally linear causal representations while also clarifying that nonlinear domains may require adaptive or learned Jacobian features.

\begin{table}[t]
\centering
\caption{Cross-set transfer stress tests.  Errors are held-out intervention-function errors; lower is better.}
\label{tab:stress-results}
\begin{tabular}{cccccc}
\toprule
Budget & \multicolumn{2}{c}{Linear funnel} & \multicolumn{3}{c}{Nonlinear funnel} \\
\cmidrule(lr){2-3}\cmidrule(lr){4-6}
& CBO & GC-CBO & CBO & GC-CBO linear & GC-CBO nonlinear \\
\midrule
$3$  & $1.5801$ & $0.7707$ & $3.9687$ & $0.5139$ & $0.4212$ \\
$6$  & $1.5136$ & $0.5312$ & $3.9317$ & $0.4875$ & $0.4724$ \\
$12$ & $1.4865$ & $0.2697$ & $3.7933$ & $0.5179$ & $0.4396$ \\
$24$ & $1.5627$ & $0.2445$ & $4.2117$ & $0.3909$ & $0.3147$ \\
$48$ & $1.6320$ & $0.1726$ & $3.6944$ & $0.3787$ & $0.3017$ \\
\bottomrule
\end{tabular}
\end{table}

\noindent\textbf{Gaussian-network case study: ECOLI70.}
We then evaluate GC-CBO on a real Gaussian Bayesian network derived from the ECOLI70 network.  This case study is important because it matches the paper narrative more closely than the nonlinear toy benchmarks: the graph is a Gaussian network, intervention effects are induced by structural equations, and multiple intervention functions share downstream causal parameters.  We consider two variants.  The parent-intervention variant is included in the main benchmark comparison in Figure~\ref{fig:benchmark-convergence} and Table~\ref{tab:main-benchmarks}, while the no-parent ancestor-intervention variant is shown separately in Figure~\ref{fig:ecoli70-noparent} and Tables~\ref{tab:ecoli70-noparent}--\ref{tab:ecoli70-ablation}.

The first variant allows direct intervention on the parents of the target.  For target \texttt{yaeM}, the optimal intervention set is
\[
        \{\texttt{cspG},\texttt{lacA},\texttt{lacZ}\},
\]
with optimal value
\[
        f^\star=-4.6304 .
\]
Both CBO and GC-CBO reach this optimum reliably, whereas vanilla BO remains far from it, as shown in the ECOLI70 \texttt{yaeM} rows of Table~\ref{tab:main-benchmarks} and the corresponding panel of Figure~\ref{fig:benchmark-convergence}.  Specifically, BO obtains terminal mean $-1.9102$, while CBO and GC-CBO both obtain $-4.6304$.  This experiment confirms that GC-CBO behaves correctly in a theorem-aligned Gaussian-network setting and that the finite-rank diagnostic again holds:
\[
        \operatorname{rank}(K_T)=16=\dim(\theta_Y).
\]
However, this setting is also too easy for CBO because the direct parent set is available and the exploration family is small.  Consequently, this experiment should be interpreted as a correctness check rather than as the main evidence for superiority over CBO.

The second ECOLI70 variant is designed to reflect a more realistic experimental constraint: direct parents of the target cannot be manipulated.  For target \texttt{b1583}, the direct parents
\[
        \{\texttt{lacA},\texttt{lacZ},\texttt{yceP}\}
\]
are excluded from the intervention set.  The allowed variables are instead non-parent ancestors, and the optimizer searches over singleton, pair, triple, and up-to-five-variable ancestor interventions, yielding a substantially larger intervention-family search problem.  This setting better exposes the difference between independent CBO and GC-CBO: CBO must learn many set-specific response surfaces, whereas GC-CBO can share information through the downstream Gaussian mechanism.  The results of this no-parent case are plotted in Figure~\ref{fig:ecoli70-noparent} and reported in Table~\ref{tab:ecoli70-noparent}.  Because vanilla BO over the full manipulative set would violate the imposed set-size cap, the no-parent ECOLI70 comparisons intentionally focus on causal methods and ablations that respect the same intervention-family constraint.

For singleton and pair interventions, the optimum is
\[
        f^\star=0.8427,
        \qquad
        \text{best set}=\{\texttt{lacY},\texttt{eutG}\}.
\]
Both CBO and GC-CBO have terminal median $0.8427$, but their terminal means are $0.8930$ due to seed variability, as reported in the max-set-size-$2$ rows of Table~\ref{tab:ecoli70-noparent}.  The corresponding convergence behavior is shown in the left panel of Figure~\ref{fig:ecoli70-noparent}.  When triple interventions are allowed, the optimal set becomes
\[
        \{\texttt{lacY},\texttt{eutG},\texttt{fixC}\},
        \qquad
        f^\star=0.6422 .
\]
In this harder setting, the no-parent sweep in Table~\ref{tab:ecoli70-noparent} shows that GC-CBO improves the terminal mean from $0.7315$ for CBO to $0.7154$ and reduces the terminal standard deviation from $0.1377$ to $0.1097$.  When the maximum intervention-set size is increased to $5$, the feasible optimum improves further to
\[
        f^\star=0.3183,
        \qquad
        \text{best set}=
        \{\texttt{lacY},\texttt{eutG},\texttt{fixC},\texttt{cspG},\texttt{sucA}\}.
\]
This final sweep is substantially harder because the best feasible intervention is now a high-order ancestor set and the admissible intervention family is much larger.  As shown in Table~\ref{tab:ecoli70-noparent}, GC-CBO reaches the optimum with terminal mean and median $0.3183$, whereas CBO remains above the optimum with terminal mean $0.7688$ and median $0.6927$.  The full convergence experiment in Table~\ref{tab:main-benchmarks} gives a consistent comparison under the same max-set-size-$5$ constraint: CBO reaches terminal mean $0.6680$, whereas GC-CBO reaches $0.3500$, with a lower terminal standard deviation of $0.0630$ compared with $0.1550$.  The gain is now much clearer and occurs in the scenario where transfer across many ancestor-intervention functions is most useful.

\begin{figure*}[t]
    \centering
    \includegraphics[width=.69\textwidth]{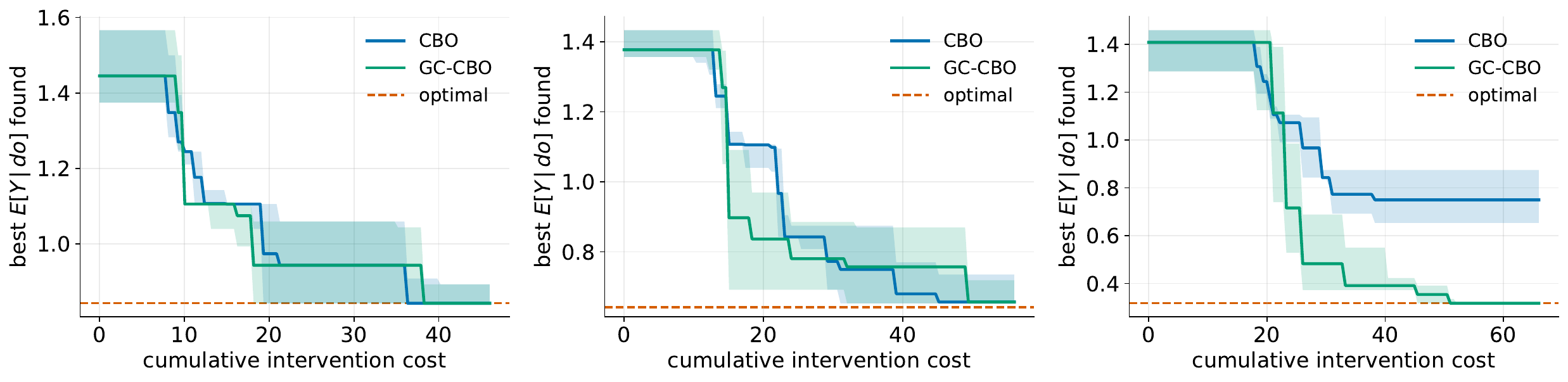}
    \includegraphics[width=.3\textwidth]{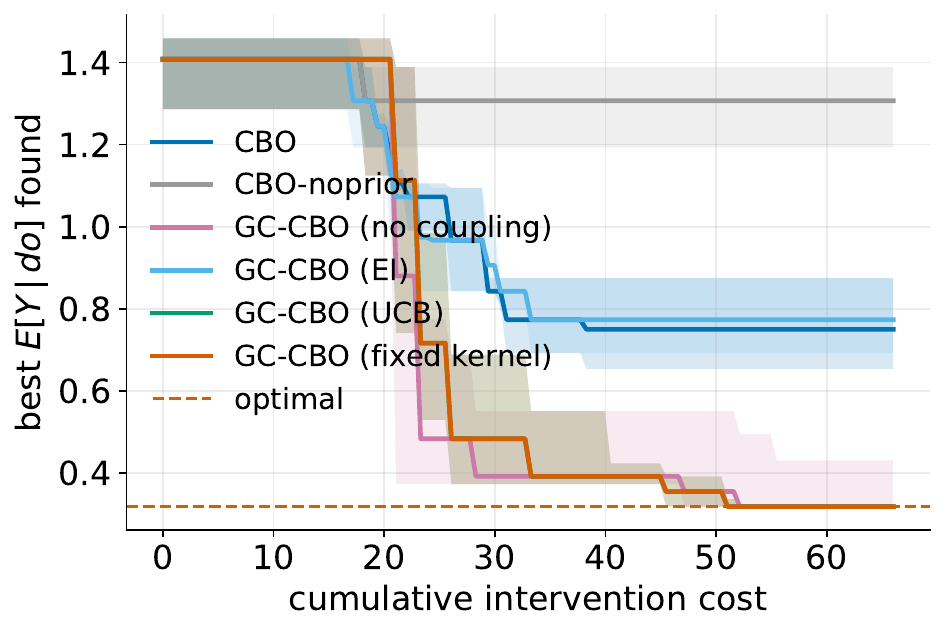}
    \caption{
    ECOLI70 no-parent intervention experiments for target \texttt{b1583}.  Direct parents are excluded from manipulation.  Left: sweep over maximum intervention-set size.  For maximum set sizes $2$ and $3$, the feasible intervention families are still small enough that both CBO and GC-CBO reach or approach the corresponding optimum.  When the maximum set size increases to $5$, the feasible optimum moves to a higher-order ancestor intervention and the intervention family becomes much larger; GC-CBO reaches the new optimum, while independent CBO does not within the tested budget.  The cumulative cost increases with maximum set size because each intervention on a set $S$ costs $|S|$, and larger caps allow the optimizer to query larger, more expensive intervention sets.  Right: ablation study in the max-set-size-$5$ setting, showing that removing causal prior information substantially degrades performance, while the UCB and fixed-kernel GC-CBO variants reach the optimum.
    }
    \label{fig:ecoli70-noparent}
\end{figure*}

\begin{table}[t]
\centering
\caption{ECOLI70 no-parent intervention sweep for target \texttt{b1583}.  Direct parents are excluded from the intervention set.  Lower values are better.}
\label{tab:ecoli70-noparent}
\begin{tabular}{llcccccc}
\toprule
Max set size & Method & Mean & Std. & Median & $f^\star$ & Best set & Rank/dim. \\
\midrule
$2$ & CBO    & $0.8930$ & $0.0872$ & $0.8427$ & $0.8427$ & \texttt{lacY;eutG} & -- \\
$2$ & GC-CBO & $0.8930$ & $0.0872$ & $0.8427$ & $0.8427$ & \texttt{lacY;eutG} & $19/19$ \\
$3$ & CBO    & $0.7315$ & $0.1377$ & $0.6571$ & $0.6422$ & \texttt{lacY;eutG;fixC} & -- \\
$3$ & GC-CBO & $0.7154$ & $0.1097$ & $0.6571$ & $0.6422$ & \texttt{lacY;eutG;fixC} & $19/19$ \\
$5$ & CBO    & $0.7688$ & $0.1723$ & $0.6927$ & $0.3183$ & \texttt{lacY;eutG;fixC;asnA} & -- \\
$5$ & GC-CBO & $0.3183$ & $0.0000$ & $0.3183$ & $0.3183$ & \texttt{lacY;eutG;fixC;cspG;sucA} & $19/19$ \\
\bottomrule
\end{tabular}
\end{table}

The ablation study clarifies which components matter most in the max-set-size-$5$ no-parent setting.  In this harder ECOLI70 setting, CBO with causal prior reaches terminal mean $0.7775$, while CBO without the causal prior remains at $1.2752$, as reported in Table~\ref{tab:ecoli70-ablation} and shown in the right panel of Figure~\ref{fig:ecoli70-noparent}.  This is the largest ablation gap and shows that causal structure is essential when direct parent interventions are unavailable.  The GC-CBO variants separate more clearly than in the max-set-size-$3$ case: GC-CBO without explicit cross-set coupling reaches $0.4315$, GC-CBO with EI reaches $0.7934$, and the UCB and fixed-kernel variants reach $0.3183$, exactly matching the optimum.  The UCB and fixed-kernel variants also achieve terminal median equal to the optimum with zero terminal variance.  These results indicate that, in the no-parent ECOLI70 experiment, the causal prior and intervention-family restriction are essential, while the coupled UCB kernel provides the strongest and most stable performance when the family becomes combinatorial.  The broader implication is that the finite-rank GC-CBO coupling becomes most valuable when many related intervention functions must be learned from sparse data, as shown more clearly by the stress-transfer experiments in Figure~\ref{fig:stress-tests} and Table~\ref{tab:stress-results}.

\begin{table}[t]
\centering
\caption{Ablation study on the ECOLI70 no-parent max-set-size-$5$ setting.  The causal prior is essential; the UCB and fixed-kernel GC-CBO variants reach the optimum.}
\label{tab:ecoli70-ablation}
\begin{tabular}{lcccc}
\toprule
Method & Mean & Std. & Median & Cost to $\epsilon$-opt. \\
\midrule
CBO & $0.7775$ & $0.1363$ & $0.7495$ & -- \\
CBO-noprior & $1.2752$ & $0.1542$ & $1.3074$ & -- \\
GC-CBO no coupling & $0.4315$ & $0.1961$ & $0.3183$ & $28.3$ \\
GC-CBO EI & $0.7934$ & $0.1227$ & $0.7734$ & -- \\
GC-CBO UCB & $0.3183$ & $0.0000$ & $0.3183$ & $33.3$ \\
GC-CBO fixed kernel & $0.3183$ & $0.0000$ & $0.3183$ & $33.3$ \\
\bottomrule
\end{tabular}
\end{table}

\noindent\textbf{Benchmark sanity checks.}
Finally, we compare BO, CBO, and GC-CBO on the toy, synthetic, and healthcare benchmarks  from \cite{aglietti2020causalbo} used as sanity checks.  These examples are less directly aligned with the linear-Gaussian theory, especially the synthetic and healthcare settings, so we use them to test robustness rather than to validate the finite-rank theory.  Their convergence behavior is shown in Figure~\ref{fig:benchmark-convergence}, and the corresponding terminal statistics are reported in Table~\ref{tab:main-benchmarks}.

On the toy chain, the optimum is approximately $-2.17$ and is achieved by intervening on $Z$.  BO jointly intervenes on $\{X,Z\}$ and incurs higher intervention cost.  It reaches only $-1.8224$ on average, while CBO reaches $-2.1687$ and GC-CBO reaches $-2.1693$, as reported in the toy-chain rows of Table~\ref{tab:main-benchmarks}.  The corresponding panel in Figure~\ref{fig:benchmark-convergence} shows that both causal methods therefore recover the optimum, but with roughly half the cumulative cost of BO.  This confirms the classical CBO message and shows that GC-CBO preserves it.

The synthetic nonlinear DAG is more challenging.  The optimum is $-2.0$, with best set $\{D,E\}$.  BO reaches $-1.9216$, CBO reaches $-1.9019$, and GC-CBO reaches $-1.9105$, as shown in Table~\ref{tab:main-benchmarks}.  The synthetic-DAG panel of Figure~\ref{fig:benchmark-convergence} shows that all three methods remain away from the continuous optimum under the tested grid and budget.  Here GC-CBO is slightly better than CBO but does not dominate BO.  This is not a contradiction of the theory because the benchmark is nonlinear and the causal coupling is only a first-order approximation.  Rather, it illustrates a boundary of the current implementation: when the structural equations are nonlinear and the finite-rank approximation is imperfect, GC-CBO should be viewed as a useful inductive bias rather than a guaranteed improvement.

The healthcare PSA benchmark is also noisy.  All methods obtain values close to the deterministic reference value $5.155$, with BO at $5.1461$, CBO at $5.1560$, and GC-CBO at $5.1488$, as reported in the healthcare rows of Table~\ref{tab:main-benchmarks}.  The healthcare panel in Figure~\ref{fig:benchmark-convergence} shows that the methods remain close throughout the tested budget range.  Since the simulator includes stochastic outcomes, sample-based terminal values can fluctuate slightly below the deterministic reference.  We therefore interpret this experiment only qualitatively: GC-CBO remains stable and competitive, but this benchmark is not the strongest evidence for the proposed finite-rank mechanism.

The ECOLI70 no-parent max-set-size-$5$ benchmark is also included in Figure~\ref{fig:benchmark-convergence} and Table~\ref{tab:main-benchmarks}.  Unlike the parent-intervention ECOLI70 setting, this benchmark excludes direct parent interventions and forces the optimizer to search through upstream ancestor sets.  CBO reaches terminal mean $0.6680$, while GC-CBO reaches $0.3500$, and GC-CBO identifies the optimal five-variable intervention set \texttt{lacY;eutG;fixC;cspG;sucA}.  This result is one of the more practically relevant case studies because it captures the intended regime of GC-CBO: direct interventions are unavailable, the intervention family is combinatorial, and related intervention functions share downstream mechanisms.

\begin{figure*}[t]
    \centering
    \includegraphics[width=.95\textwidth]{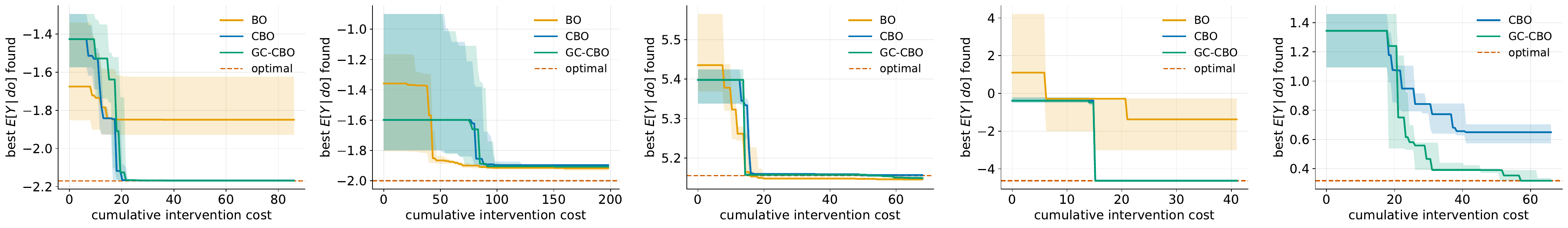}
    \caption{
    Convergence across benchmark scenarios: toy chain graph, synthetic DAG from \cite{aglietti2020causalbo}, healthcare PSA from \cite{aglietti2020causalbo}, ECOLI70 from \cite{scutari2010learning} with yaeM and b1583 as target variable, respectively.  GC-CBO preserves the gains of CBO on the toy chain, matches CBO on ECOLI70 parent interventions, remains competitive on nonlinear and noisy benchmarks, and improves over CBO in the no-parent ECOLI70 max-set-size-$5$ ancestor-intervention setting.
    }
    \label{fig:benchmark-convergence}
\end{figure*}

\begin{table*}[t]
\centering
\caption{Main benchmark results.  Lower values are better.  The strongest evidence for GC-CBO is the theory-aligned Gaussian setting and the transfer stress tests; the nonlinear and healthcare benchmarks are included as robustness checks.}
\label{tab:main-benchmarks}
\begin{tabular}{llccccc}
\toprule
Experiment & Method & Mean & Std. & Median & Total cost & Optimum \\
\midrule
Toy chain & BO & $-1.8224$ & $0.2357$ & $-1.8495$ & $86.0$ & $-2.1700$ \\
Toy chain & CBO & $-2.1687$ & $0.0009$ & $-2.1685$ & $46.0$ & $-2.1700$ \\
Toy chain & GC-CBO & $-2.1693$ & $0.0005$ & $-2.1692$ & $46.0$ & $-2.1700$ \\
\midrule
Synthetic DAG & BO & $-1.9216$ & $0.0139$ & $-1.9177$ & $198.0$ & $-2.0000$ \\
Synthetic DAG & CBO & $-1.9019$ & $0.0219$ & $-1.8971$ & $155.6$ & $-2.0000$ \\
Synthetic DAG & GC-CBO & $-1.9105$ & $0.0152$ & $-1.9072$ & $158.4$ & $-2.0000$ \\
\midrule
Healthcare PSA & BO & $5.1461$ & $0.0045$ & $5.1459$ & $66.0$ & $5.1550$ \\
Healthcare PSA & CBO & $5.1560$ & $0.0089$ & $5.1568$ & $65.6$ & $5.1550$ \\
Healthcare PSA & GC-CBO & $5.1488$ & $0.0053$ & $5.1488$ & $66.9$ & $5.1550$ \\
\midrule
ECOLI70 \texttt{yaeM} & BO & $-1.9102$ & $1.8054$ & $-1.3749$ & $33.0$ & $-4.6304$ \\
ECOLI70 \texttt{yaeM} & CBO & $-4.6304$ & $0.0000$ & $-4.6304$ & $38.6$ & $-4.6304$ \\
ECOLI70 \texttt{yaeM} & GC-CBO & $-4.6304$ & $0.0000$ & $-4.6304$ & $39.4$ & $-4.6304$ \\
\midrule
ECOLI70 \texttt{b1583} no-parent & CBO & $0.6680$ & $0.1550$ & $0.6500$ & $39.8$ & $0.3183$ \\
ECOLI70 \texttt{b1583} no-parent & GC-CBO & $0.3500$ & $0.0630$ & $0.3183$ & $63.4$ & $0.3183$ \\
\bottomrule
\end{tabular}
\end{table*}

Overall, the experiments support three conclusions.  First, in the linear-Gaussian setting where the assumptions are satisfied, GC-CBO produces the predicted finite-rank kernel, exact cross-intervention covariance, reduced information gain, and parameter concentration, as shown in Figure~\ref{fig:theory-diagnostics} and Table~\ref{tab:theory-summary}.  This is the strongest evidence that the proposed construction is mathematically and computationally faithful to the theory.  Second, in cross-set transfer problems, GC-CBO substantially improves over independent CBO because it shares information across intervention functions, as demonstrated in Figure~\ref{fig:stress-tests} and Table~\ref{tab:stress-results}.  This is the main practical advantage of the method.  Third, on small causal benchmarks, GC-CBO often matches rather than dramatically exceeds CBO, as reflected in Figure~\ref{fig:benchmark-convergence}, Table~\ref{tab:main-benchmarks}, Figure~\ref{fig:ecoli70-noparent}, and Tables~\ref{tab:ecoli70-noparent}--\ref{tab:ecoli70-ablation}.  This is expected: when the correct parent set is directly manipulable or the intervention family is small, independent CBO already has enough information.  The benefit of GC-CBO becomes most meaningful when direct interventions are unavailable, the exploration family is combinatorial, or sparse interventional data must be shared across related causal mechanisms.  The updated max-set-size-$5$ ECOLI70 result reinforces this interpretation: once the optimizer must search over upstream ancestor interventions of increasing order, GC-CBO improves both mean performance and stability relative to independent CBO, while the ablation study confirms that the causal prior and coupled UCB kernel are essential for making the no-parent search tractable.

\section{Conclusion}
\label{sec:conclusion}
We introduced graph-coupled causal Bayesian optimization, which ties the effects of different interventions together through the posterior uncertainty of a shared, identifiable causal parameterization. The resulting causal kernel is exact and finite-rank in identifiable linear Gaussian models, yielding a logarithmic information-gain bound and a regret decomposition into optimization, causal-estimation, and exploration-set terms. Nonlinear and adaptive extensions retain the graph-factorized principle under clearly stated, weaker guarantees. Empirically, the method preserves the strengths of causal BO and adds genuine transfer across related interventions, most visibly when direct parent interventions are off-limits and interventional data are sparse. Natural next steps include sharper analysis of the adaptive, time-varying kernel and learned causal features for strongly nonlinear mechanisms.

\bibliographystyle{plainnat}  
\bibliography{mybibfile}

@book{horn2013matrix,
  title={Matrix Analysis},
  author={Horn, Roger A. and Johnson, Charles R.},
  edition={2},
  year={2013},
  publisher={Cambridge University Press}
}

@inproceedings{aglietti2020causalbo,
  title={Causal {B}ayesian Optimization},
  author={Aglietti, Virginia and Lu, Xiaoyu and Paleyes, Andrei and Gonz{\'a}lez, Javier},
  booktitle={Proceedings of the 23rd International Conference on Artificial Intelligence and Statistics},
  pages={3155--3164},
  year={2020},
  organization={PMLR}
}

@book{pearl2009causality,
  title={Causality: Models, Reasoning, and Inference},
  author={Pearl, Judea},
  edition={2},
  year={2009},
  publisher={Cambridge University Press}
}

@inproceedings{shpitser2006identification,
  title={Identification of Joint Interventional Distributions in Recursive Semi-Markovian Causal Models},
  author={Shpitser, Ilya and Pearl, Judea},
  booktitle={Proceedings of the Twenty-First National Conference on Artificial Intelligence},
  year={2006}
}

@article{srinivas2012information,
  title={Information-Theoretic Regret Bounds for {G}aussian Process Optimization in the Bandit Setting},
  author={Srinivas, Niranjan and Krause, Andreas and Kakade, Sham M. and Seeger, Matthias},
  journal={IEEE Transactions on Information Theory},
  volume={58},
  number={5},
  pages={3250--3265},
  year={2012}
}

@book{rasmussen2006Gaussian,
  title={{G}aussian Processes for Machine Learning},
  author={Rasmussen, Carl Edward and Williams, Christopher K. I.},
  year={2006},
  publisher={MIT Press}
}

@inproceedings{vakili2021information,
  title={On Information Gain and Regret Bounds in {G}aussian Process Bandits},
  author={Vakili, Sattar and Khezeli, Kia and Picheny, Victor},
  booktitle={Proceedings of the 24th International Conference on Artificial Intelligence and Statistics},
  year={2021},
  organization={PMLR}
}

@article{shahriari2016taking,
  title={Taking the human out of the loop: A review of {B}ayesian optimization},
  author={Shahriari, Bobak and Swersky, Kevin and Wang, Ziyu and Adams, Ryan P and De Freitas, Nando},
  journal={Proceedings of the IEEE},
  volume={104},
  number={1},
  pages={148--175},
  year={2015},
  publisher={IEEE}
}

@article{frazier2018tutorial,
  title         = {A Tutorial on {B}ayesian optimization},
  author        = {Frazier, Peter I.},
  journal       = {arXiv preprint arXiv:1807.02811},
  year          = {2018},
  eprint        = {1807.02811},
  archivePrefix = {arXiv}
}

@book{garnett2023bayesian,
  title     = {{B}ayesian optimization},
  author    = {Garnett, Roman},
  year      = {2023},
  publisher = {Cambridge University Press}
}

@article{daulton2020differentiable,
  title={Differentiable expected hypervolume improvement for parallel multi-objective {B}ayesian optimization},
  author={Daulton, Samuel and Balandat, Maximilian and Bakshy, Eytan},
  journal={Advances in neural information processing systems},
  volume={33},
  pages={9851--9864},
  year={2020}
}

@article{daulton2021parallel,
  title={Parallel {B}ayesian optimization of multiple noisy objectives with expected hypervolume improvement},
  author={Daulton, Samuel and Balandat, Maximilian and Bakshy, Eytan},
  journal={Advances in neural information processing systems},
  volume={34},
  pages={2187--2200},
  year={2021}
}

@inproceedings{eriksson2021highdimensional,
  title={High-dimensional {B}ayesian optimization with sparse axis-aligned subspaces},
  author={Eriksson, David and Jankowiak, Martin},
  booktitle={Uncertainty in artificial intelligence},
  pages={493--503},
  year={2021},
  organization={PMLR}
}

@inproceedings{astudillo2022thinking,
  title={Thinking inside the box: A tutorial on grey-box {B}ayesian optimization},
  author={Astudillo, Raul and Frazier, Peter I},
  booktitle={2021 Winter Simulation Conference (WSC)},
  pages={1--15},
  year={2021},
  organization={IEEE}
}

@article{xie2024costaware,
  title={Cost-aware {B}ayesian optimization via the pandora's box gittins index},
  author={Xie, Qian and Astudillo, Raul and Frazier, Peter I and Scully, Ziv and Terenin, Alexander},
  journal={Advances in Neural Information Processing Systems},
  volume={37},
  pages={115523--115562},
  year={2024}
}

@article{xie2024globalacq,
  title         = {Global Optimization of {G}aussian Process Acquisition Functions Using a Piecewise-Linear Kernel Approximation},
  author        = {Xie, Yilin and Zhang, Shiqiang and Paulson, Joel A. and Tsay, Calvin},
  journal       = {arXiv preprint arXiv:2410.16893},
  year          = {2024},
  eprint        = {2410.16893},
  archivePrefix = {arXiv}
}

@inproceedings{branchini2022causal,
  title={Causal entropy optimization},
  author={Branchini, Nicola and Aglietti, Virginia and Dhir, Neil and Damoulas, Theodoros},
  booktitle={International Conference on Artificial Intelligence and Statistics},
  pages={8586--8605},
  year={2023},
  organization={PMLR}
}

@inproceedings{aglietti2023constrained,
  title={Constrained causal {B}ayesian optimization},
  author={Aglietti, Virginia and Malek, Alan and Ktena, Ira and Chiappa, Silvia},
  booktitle={International Conference on Machine Learning},
  pages={304--321},
  year={2023},
  organization={PMLR}
}

@inproceedings{gultchin2023functional,
  title={Functional causal {B}ayesian optimization},
  author={Gultchin, Limor and Aglietti, Virginia and Bellot, Alexis and Chiappa, Silvia},
  booktitle={Uncertainty in Artificial Intelligence},
  pages={756--765},
  year={2023},
  organization={PMLR}
}

@inproceedings{sussex2023adversarial,
  title={Adversarial causal {B}ayesian optimization},
  author={Sussex, Scott and Sessa, Pier Giuseppe and Makarova, Anastasia and Krause, Andreas},
  booktitle={International Conference on Learning Representations},
  volume={2024},
  pages={19332--19353},
  year={2024}
}

@article{ren2024exogenous,
  title         = {Causal {B}ayesian optimization via Exogenous Distribution Learning},
  author        = {Ren, Shaogang and Qian, Xiaoning},
  journal       = {arXiv preprint arXiv:2402.02277},
  year          = {2024},
  eprint        = {2402.02277},
  archivePrefix = {arXiv}
}

@article{bogunovic2021misspecified,
  title={Misspecified {G}aussian process bandit optimization},
  author={Bogunovic, Ilija and Krause, Andreas},
  journal={Advances in neural information processing systems},
  volume={34},
  pages={3004--3015},
  year={2021}
}

@inproceedings{kassraie2021neural,
  title={Neural contextual bandits without regret},
  author={Kassraie, Parnian and Krause, Andreas},
  booktitle={International Conference on Artificial Intelligence and Statistics},
  pages={240--278},
  year={2022},
  organization={PMLR}
}

@inproceedings{tran2022regret,
  title={Regret bounds for expected improvement algorithms in {G}aussian process bandit optimization},
  author={Gupta, Sunil and Rana, Santu and Venkatesh, Svetha and others},
  booktitle={International Conference on Artificial Intelligence and Statistics},
  pages={8715--8737},
  year={2022},
  organization={PMLR}
}

@article{iwazaki2026improved,
  title={Improved regret bounds for {G}aussian process upper confidence bound in {B}ayesian optimization},
  author={Iwazaki, Shogo},
  journal={Advances in Neural Information Processing Systems},
  volume={38},
  pages={96922--96964},
  year={2026}
}

@article{scutari2010learning,
  title={Learning Bayesian networks with the bnlearn R package},
  author={Scutari, Marco},
  journal={Journal of statistical software},
  volume={35},
  pages={1--22},
  year={2010}
}

@inproceedings{pfisterer2022yahpo,
  title={Yahpo gym-an efficient multi-objective multi-fidelity benchmark for hyperparameter optimization},
  author={Pfisterer, Florian and Schneider, Lennart and Moosbauer, Julia and Binder, Martin and Bischl, Bernd},
  booktitle={International Conference on Automated Machine Learning},
  pages={3--1},
  year={2022},
  organization={PMLR}
}

@article{hickman2025anubis,
  title={Anubis: Bayesian optimization with unknown feasibility constraints for scientific experimentation},
  author={Hickman, Riley J and Tom, Gary and Zou, Yunheng and Aldeghi, Matteo and Aspuru-Guzik, Al{\'a}n},
  journal={Digital Discovery},
  volume={4},
  number={8},
  pages={2104--2122},
  year={2025},
  publisher={Royal Society of Chemistry}
}

@article{liang2021benchmarking,
  title={Benchmarking the performance of {B}ayesian optimization across multiple experimental materials science domains},
  author={Liang, Qiaohao and Gongora, Aldair E and Ren, Zekun and Tiihonen, Armi and Liu, Zhe and Sun, Shijing and Deneault, James R and Bash, Daniil and Mekki-Berrada, Flore and Khan, Saif A and others},
  journal={npj Computational Materials},
  volume={7},
  number={1},
  pages={188},
  year={2021},
  publisher={Nature Publishing Group UK London}
}

@incollection{frazier2015bayesian,
  title={Bayesian optimization for materials design},
  author={Frazier, Peter I and Wang, Jialei},
  booktitle={Information science for materials discovery and design},
  pages={45--75},
  year={2015},
  publisher={Springer}
}

@article{do2023multi,
  title={Multi-fidelity {B}ayesian optimization in engineering design},
  author={Do, Bach and Zhang, Ruda},
  journal={arXiv preprint arXiv:2311.13050},
  year={2023}
}

@article{hossen2026multi,
  title={Multi-Objective Multi-Fidelity {B}ayesian Optimization with Causal Priors},
  author={Hossen, Md Abir and Javidian, Mohammad Ali and Narayanan, Vignesh and O'Kane, Jason M and Jamshidi, Pooyan},
  journal={arXiv preprint arXiv:2602.00788},
  year={2026}
}

@inproceedings{roberts2024causalbo,
  title={CausalBO: A Python Package for Causal {B}ayesian Optimization},
  author={Roberts, Jeremy and Javidian, Mohammad Ali},
  booktitle={SoutheastCon 2024},
  pages={1370--1375},
  year={2024},
  organization={IEEE}
}

@article{jacobs2026extending,
  title={Extending Multi-Source {B}ayesian Optimization With Causality Principles},
  author={Jacobs, Luuk and Javidian, Mohammad Ali},
  journal={arXiv preprint arXiv:2602.14791},
  year={2026}
}

\end{document}